\documentclass[lettersize,journal]{IEEEtran}
\usepackage{amsmath,amsfonts,amsthm}
\usepackage{algorithmic}
\usepackage{algorithm}
\usepackage{array}
\usepackage[caption=false,font=normalsize,labelfont=sf,textfont=sf]{subfig}
\usepackage{textcomp}
\usepackage{stfloats}
\usepackage{url}
\usepackage{verbatim}
\usepackage{graphicx}
\usepackage{cite}
\usepackage{booktabs}
\usepackage{multirow}
\usepackage{makecell}
\usepackage{color}
\usepackage{CJKutf8}
\usepackage{pifont}
\usepackage{colortbl}
\newcolumntype{H}{>{\setbox0=\hbox\bgroup}c<{\egroup}@{}}
\usepackage{bm}
\usepackage{bbm}
\usepackage{pythonhighlight}
\graphicspath{{./Figure/}}
\usepackage{mathrsfs}
\usepackage{mathrsfs}
\usepackage{MnSymbol}
\usepackage{dashrule}
\usepackage{tikz}
\usepackage{mathtools}
\usepackage{amsthm}
\newtheorem{definition}{Definition} 

\newtheorem{theorem}{Theorem}

\usepackage{adjustbox}

\definecolor{orange}{RGB}{247, 224, 213}

\begin{document}

\title{All-in-One Image Restoration via Causal-Deconfounding Wavelet-Disentangled Prompt Network}

\author{Bingnan~Wang, Bin~Qin, Jiangmeng~Li, Fanjiang~Xu, Fuchun~Sun,~\IEEEmembership{Fellow,~IEEE}, and Hui~Xiong,~\IEEEmembership{Fellow,~IEEE}

\thanks{This work is supported by National Natural Science Foundation of China No. 62406313, Postdoctoral Fellowship Program of China Postdoctoral Science Foundation, Grant No. YJB20250283. (\textit{Bingnan Wang, Bin Qin, and Jiangmeng Li have contributed equally to this work. Corresponding author: Jiangmeng Li.})}
\thanks{Bingnan Wang, Bin Qin, Jiangmeng Li, and Fanjiang Xu are with University of Chinese Academy of Sciences, Beijing 100190, China, and also with Institute of Software, Chinese Academy of Science, Beijing 100190, China (e-mail: wangbingnan21@mails.ucas.ac.cn, qinbin21@mails.ucas.ac.cn, jiangmeng2019@iscas.ac.cn, fanjiang@iscas.ac.cn).}% <-this % stops a space
\thanks{Fuchun Sun is with the Department of Computer Science
and Technology, Tsinghua University, Beijing National Research Center
for Information Science and Technology (BNRist), Beijing 100084, China
(e-mail: fcsun@tsinghua.edu.cn).}
\thanks{Hui Xiong is with the Artificial Intelligence Thrust, The Hong Kong
University of Science and Technology (Guangzhou), Guangzhou 511458,
China (e-mail: xionghui@ust.hk).}}

% The paper headers
\markboth{Journal of \LaTeX\ Class Files,~Vol.~14, No.~8, August~2021}%
{Shell \MakeLowercase{\textit{et al.}}: A Sample Article Using IEEEtran.cls for IEEE Journals}

% \IEEEpubid{0000--0000/00\$00.00~\copyright~2021 IEEE}
% Remember, if you use this you must call \IEEEpubidadjcol in the second
% column for its text to clear the IEEEpubid mark.

\maketitle

\begin{abstract}
Image restoration represents a promising approach for addressing the inherent defects of image content distortion. Standard image restoration approaches suffer from \textit{high storage cost} and the requirement towards the \textit{known} degradation pattern, including type and degree, which can barely be satisfied in dynamic practical scenarios. In contrast, all-in-one image restoration (AiOIR) eliminates multiple degradations within a unified model to circumvent the aforementioned issues. However, according to our causal analysis, we disclose that two significant defects still exacerbate the effectiveness and generalization of AiOIR models: 1) the \textit{spurious correlation} between \textit{non-degradation} semantic features and degradation patterns; 2) the \textit{biased estimation} of degradation patterns. To obtain the true causation between degraded images and restored images, we propose \textit{\textbf{C}ausal-deconfounding \textbf{W}avelet-disentangled \textbf{P}rompt \textbf{Net}work} (CWP-Net) to perform effective AiOIR. CWP-Net introduces two modules for decoupling, i.e., wavelet attention module of encoder and wavelet attention module of decoder. These modules explicitly disentangle the degradation and semantic features to tackle the issue of spurious correlation. To address the issue stemming from the biased estimation of degradation patterns, CWP-Net leverages a wavelet prompt block to generate the alternative variable for causal deconfounding. Extensive experiments on two all-in-one settings prove the effectiveness and superior performance of our proposed CWP-Net over the state-of-the-art AiOIR methods.
\end{abstract}


 \section{Introduction} \label{intro}
The inherent defects of the imaging system (e.g., noise) and poor environmental conditions (e.g., rain and low light) can lead to distortion of the image content. A dominant choice for addressing this issue is image restoration, which aims to restore clean images from degraded counterparts through \textit{post-processing}. Image restoration demonstrates widely promising performance in the fields of video surveillance \cite{li2021online}, autonomous driving \cite{li2021deep}, and mobile phone imaging \cite{lai2022face}.

The canonical image restoration approaches contain two learning paradigms: 1) the task-specific image restoration methods \cite{li2019heavy,zheng2023curricular,wei2018deep,zhang2021exposure}, utilizing the degradation mechanisms associated with specific tasks (such as denoising, deblurring, etc.) as priors to guide the design of network architectures. 2) the general image restoration methods \cite{li2023efficient,fei2023generative,guo2024mambair}, leveraging a shared network architecture while training separate model weights for different tasks. Inevitably, the canonical image restoration approaches encounter two major challenges in dynamic practical scenarios: \textit{high storage cost} and the requirement for the \textit{known} degradation pattern (type and degree) to retrieve the corresponding model weights.

In contrast, all-in-one image restoration (AiOIR) aims to eliminate multiple degradations within a single model \cite{zhu2023learning,chen2021pre,li2020all} to alleviate the \textit{high storage cost} issue. A prevalent strategy in AiOIR models is to predict degradation patterns, which subsequently guide the modulation of the restoration network's parameter space for diverse tasks. Recent advancements have explored both explicit \cite{park2023all,luo2023controlling,li2022all,li2023prompt} and implicit \cite{potlapalli2024promptir} learning of degradation patterns by incorporating auxiliary degradation classifiers (e.g., CNNs or large-scale vision-language models). In this regard, we model the inherent causal mechanism behind AiOIR by leveraging the structural causal model (SCM) \cite{pearl2009causality}. From the causal perspective, the objective of AiOIR models is to learn the deconfounded \textit{causation} between the degraded image and the restored image. However, we disclose that current methods still encounter two significant defects in deconfounding, thereby halting the rise of the effectiveness and generalization: 1) the \textit{spurious correlation} between \textit{non-degradation} semantic features and degradation patterns; 2) the \textit{biased estimation} of degradation patterns. Detailed causal analysis can be seen in Sec. \ref{causal}.

\begin{figure*}[htbp]
\centering
\vspace{-0.5cm}
\includegraphics[width=\linewidth]{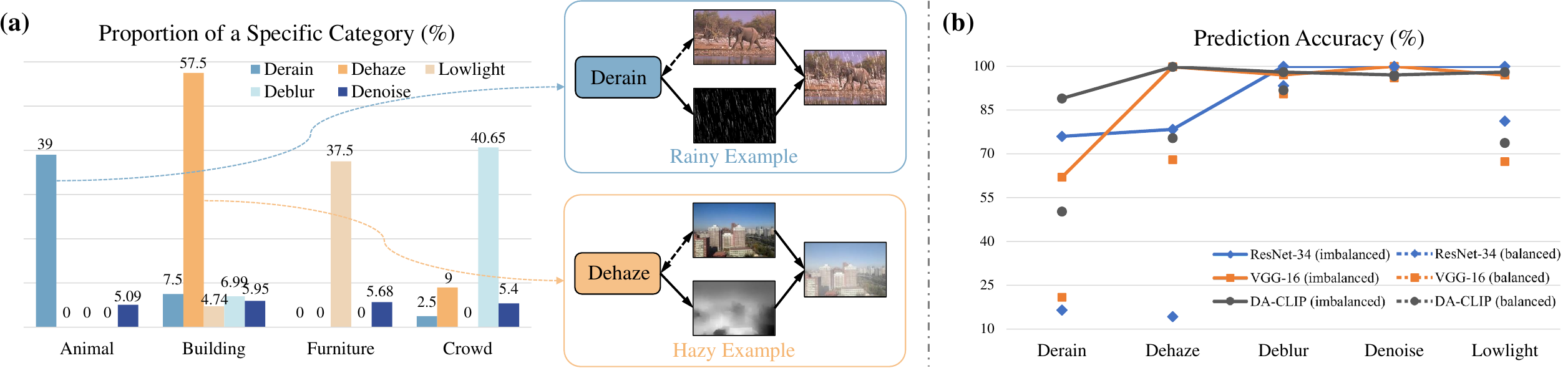}
\vspace{-0.5cm}
\caption{(a) The illustration of spurious correlation. We present the statistical analysis of four selected categories in the training data of five tasks. Detailed dataset information is provided in Sec. \ref{sec:dataset}. (b) The prediction accuracy of five degradation patterns on imbalanced test data and balanced test data. ResNet-34 \cite{he2016deep}, VGG-16 \cite{simonyan2014very}, and DA-CLIP \cite{luo2023controlling} are utilized as degradation classifiers. The training and testing data used in conventional AiOIR models are considered imbalanced. Balanced means that
the task data is not limited to the scenarios present in the
training set, but rather it is rich and comprehensive. All classifiers are trained on imbalanced training data. The balanced test data is constructed with details outlined in Sec. \ref{sec:exp_unbiased}.}
\vspace{-0.3cm}
\label{fig_sc}
\end{figure*}

As demonstrated in Fig. \ref{fig_sc}(a), we provide intuitive examples and statistical evidence to prove the general existence of spurious correlation. For example, animal-related scenes constitute 39\% of the samples in the deraining dataset; however, such scenes are absent from deblurring and dehazing datasets. Consequently, the degradation classifiers may learn the unexpected spurious correlation between degradation patterns and non-degradation semantic features, resulting in detrimental impacts on AiOIR models. Furthermore, the experiments in Fig. \ref{fig_sc} (b) demonstrate the issue of biased estimation of degradation patterns. Even on imbalanced data, these models fail to achieve perfect predictions for tasks such as deraining and dehazing. The insurmountable gap between performance on balanced and imbalanced test sets substantiates that the presence of spurious correlations further exacerbates the estimation bias. Therefore, enabling AiOIR models to effectively address these limitations and obtain true causal effects poses challenges.

To this end, we propose \textit{\textbf{C}ausal-deconfounding \textbf{W}avelet-disentangled \textbf{P}rompt \textbf{Net}work} (CWP-Net) to perform effective AiOIR. To tackle the issue of spurious correlation, two modules for decoupling are inserted in the encoder and decoder: wavelet attention module of encoder (WAE) and wavelet attention module of decoder (WAD). The attention maps of the approximate wavelet coefficients are employed as the degradation representations. These maps explicitly disentangle the degradation features from the semantic features by assigning significant weights to the degraded regions and small weights to the semantic regions. The mitigation of spurious correlations thereby facilitates the precise inference of causal effects. To alleviate the adverse effects on AiOIR models by biased estimation of degradation patterns, we explore the alternative variable to perform sufficient causal deconfounding. Specifically, we propose a wavelet prompt block (WPB), comprising two components: the degradation-based weight estimator (DWE) and the prompt-guided weighted spatial feature transform (PWSFT), which work together to derive the alternative variable for causal deconfounding and to implement backdoor adjustment. Empirical experiments prove that the proposed CWP-Net can significantly outperform the state-of-the-art AiOIR methods. To summarize, our \textit{contributions} are four-fold:

\begin{itemize}
\item We elaborate on the intrinsic reasons exacerbating the performance of AiOIR methods: 1) the spurious correlation between semantic features and degradation patterns; 2) the biased estimation of degradation patterns.
\item We present the SCM-based analysis to theoretically model the causal mechanism of AiOIR, whose correctness is validated through causal discovery theory, and further provide a causality-guided methodology that directs the module design to address the inherent defects of existing methods.
\item We propose CWP-Net, incorporating two wavelet attention modules to mitigate spurious correlation, and the wavelet prompt block to explore the deconfounded causation between the input degraded image and restored image, thereby improving the model effectiveness.
\item Extensive experiments demonstrate that the proposed CWP-Net can improve the performance of AiOIR models on five benchmarks with two all-in-one settings and exhibits superior generalization over existing methods.
\end{itemize}

\section{Related Work}
\subsection{Canonical Image Restoration}

As a long-standing problem, image restoration aims to recover high-quality images from degraded ones. Due to the strong non-linearity and generalization of deep networks, deep learning-based methods have become the mainstream paradigm in the field \cite{wang2024survey}. For different image restoration tasks (such as deraining \cite{refr3.4}, motion deblurring \cite{refr3.1,refr3.2}, low-light enhancement \cite{refr3.6}, etc.), various carefully-designed networks have been proposed to remove specific degradation \cite{chen2023learning,song2023vision}. For example, Zhang et al. \cite{refr3.6} link semantic segmentation and depth estimation to the snow removal task, integrating semantic and geometric priors to improve desnowing performance.

Recently, the research focus has shifted towards designing generic restoration models \cite{zamir2022restormer,li2023efficient,cui2023focal} that can achieve superior performance in various tasks. For instance, Restormer \cite{zamir2022restormer} proposes an efficient transformer architecture with multi-axis MLPs to handle high-resolution image restoration, achieving state-of-the-art performance while maintaining computational efficiency. Similarly, MB-TaylorFormer V2 \cite{refr3.5} introduces an improved multi-branch linear transformer expanded by the Taylor formula, which enhances model expressiveness and generalization for diverse image restoration tasks. However, its performance in real-world scenarios is limited in two aspects: first, we need to train and store distinct models for each task, which is computationally intensive and storage-consuming. Second, the need for prior information of the degradation type and degree is inaccessible.

\subsection{All-in-One Image Restoration}
AiOIR methods conduct multiple tasks (such as deraining, denoising, etc.) using a single model. Early models assigned a separate feature extractor for each task as encoders or decoders \cite{li2020all, chen2021pre}. Additionally, some methods explicitly separated task-general and task-specific features to facilitate multi-degradation learning \cite{zhu2023learning}. These methods violate the strict all-in-one setting, where the degradation pattern is blind to the model during inference. 

In the canonical image restoration, enhancing model robustness in real-world scenarios has motivated two major lines of research. The first incorporates high-level semantic priors, leveraging the scene understanding capability of large-scale vision foundation models to provide contextual guidance for pixel-level restoration \cite{xiao2023dive}. The second follows an explicit degradation estimation paradigm \cite{xiao2025incorporating,refr3.6}, in which a lightweight auxiliary network is developed to provide degradation guidance information adaptively.

Inspired by these advances, a natural extension is to transfer such ideas to the AiOIR setting, where degradation types and intensities are more diverse and intertwined. Accordingly, unified AiOIR approaches \cite{refr3.7} have been proposed following the design framework of degradation representation and degradation-guided restoration. During the degradation representation stage, techniques such as contrastive learning \cite{li2022all}, prompt learning \cite{potlapalli2024promptir}, and large-scale vision-language models (e.g., CLIP) \cite{luo2023controlling} are introduced to learn degradation patterns explicitly or implicitly. Subsequently, the degraded representation modulates the features of the restoration branch to facilitate task-relevant feature learning. However, the spurious correlation between degradation patterns and semantic features learned by the model and the biased estimation of degradation patterns exacerbate the generalization and effectiveness of these all-in-one models.

\subsection{Wavelet-based Image Processing}
As an effective frequency analysis tool, wavelet transform can extract low- and high-frequency information explicitly. It has been introduced into the network design to enhance local details in various image restoration tasks \cite{zou2021sdwnet, zhang2023underwater}. MRLPFNet \cite{dong2023multi} integrated wavelet transform in skip-connection to fuse multi-scale encoder features for image deblurring. \cite{zhang2024pan} decomposed features into four subbands and refined each subband for pan-sharpening. Moreover, significant efforts have been made in various domains, such as super-resolution \cite{huang2017wavelet}, low-light enhancement \cite{xu2022illumination}, and image inpainting \cite{yu2021wavefill}, to design wavelet-based methods for improving model performance. In contrast, we first introduce the wavelet transform into the AiOIR task, leveraging it as the primary tool to remedy the defect of spurious correlations and to achieve causal deconfounding.

\section{Theoretical Analysis with Causal Insights}
\label{causal}
In this section, we first provide an SCM analysis of the AiOIR problem. Subsequently, we explain the challenges of existing AiOIR methods for sufficient causal deconfounding in the SCM framework.

\subsection{Causal Modeling for All-in-One Image Restoration}
\label{sec:causal model}

We have modeled a unified SCM framework for the causal mechanism behind \textit{ideal AiOIR}, which is depicted in Fig. \ref{fig_scm}(a). We will explain each path individually: 1) $X \rightarrow R \rightarrow Y$ illustrates the true causation expected to be learned by the restoration network $f$. 2) $C \rightarrow X \leftarrow D \leftarrow T$ represents the generation process of the degraded images. The training set of a specific $T$ comprises degraded images with various degradation features. Notably, in the \textit{ideal AiOIR} model, the semantic feature $C$ can be regarded as a random background variable that does not form a confounding path for the causation of $X$ on $Y$. 3) $T \rightarrow Y$ represents the degradation-specific knowledge learned from each degradation pattern to the corresponding restored image. The confounding path $X \leftarrow D \leftarrow T \rightarrow Y$ causes the AiOIR network $f$ to learn the correlation $\mathbb{P}(Y|X)$ instead of the causation of $X$ on $Y$ (i.e., $\mathbb{P}(Y|do(X))$). Therefore, from the causal perspective, the goal of AiOIR models $f$ is to achieve causal deconfounding, thereby facilitating the learning for the causation of $X$ on $Y$.

\begin{figure}[htbp]
\centering
\includegraphics[width=0.8\linewidth]{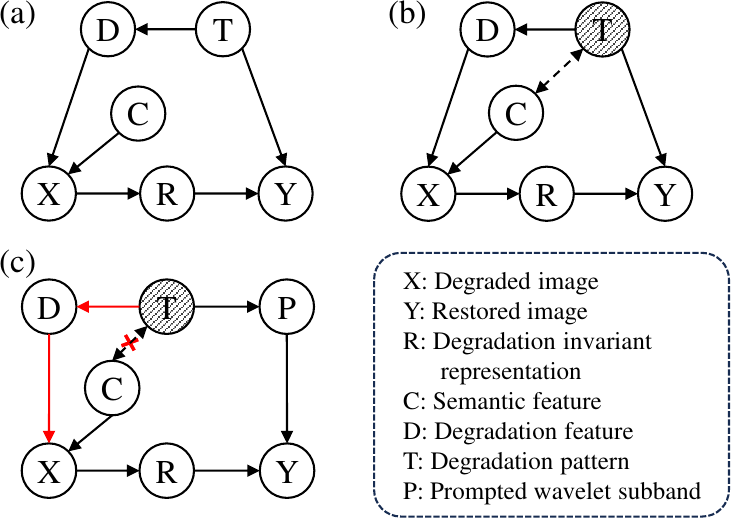}
\vspace{-0.3cm}
\caption{The SCM graph of AiOIR. (a) A unified causal framework for existing \textit{ideal AiOIR} methods.
% : there only exists one confounding path $X \leftarrow D \leftarrow T \rightarrow Y$ satisfying $C \upmodels T$. 
(b) Two challenges of \textit{practical AiOIR} via causal lens. The dashed line $\dashleftarrow \dashrightarrow $ represents the spurious correlation, and the grey node denotes the unobservability of the variable. (c) Causal deconfounding through alternative variable $P$. The \textcolor{red}{red} path and the symbol ``$\textcolor{red}{\times}$'' indicate the use of the wavelet attention module to eliminate the spurious correlation between $C$ and $T$, allowing the restoration network to focus on degraded regions rather than semantic features.}
\vspace{-0.3cm}
\label{fig_scm}
\end{figure}

\subsection{Challenges of Practical AiOIR via Causal Lens}
\label{sec:cau_challenge}
Most efforts in AiOIR models can be summarized as mitigating the negative impact of the confounder $T$ in Fig. \ref{fig_scm}(a). For example, previous methods first estimate $T$ of the degraded image \cite{park2023all,luo2023controlling} and then adapt the restoration network to handle different types and degrees of degradation. From a causal perspective, this can be understood as conditioning on the variable $T$, such that $X$ and $Y$ are then $d$-separated \cite{pearl2009causality,pearl2016causal} with only the confounding path considered. In other words, the confounding path $X \leftarrow D \leftarrow T \rightarrow Y$ is blocked by variable $T$, thereby the causation of $X$ on $Y$ can be learned by models. However, existing methods still encounter two challenges in practical scenarios as stated in Sec. \ref{intro}. Here, we further explain the detrimental impact of these two issues on AiOIR models for causal deconfounding.

\textbf{Spurious correlation challenge.} In causality, the spurious correlation is represented by the dashed line $C \dashleftarrow \dashrightarrow T$ as shown in Fig. \ref{fig_scm}(b), which can be generally explained by a latent common cause \cite{verma1993graphical}: $C \leftarrow E \rightarrow T$, where $E$ denotes some unobservable hidden variables. For example, a particular value $E=e$ might indicate that furniture often appears in low-light environments or that animals often appear with rain streaks. The presence of spurious correlations leads to estimated values of the confounder (e.g., $T$) being correlated with both $D$ and $C$. In the balanced test set, where the spurious correlation of $C$ and $T$ is absent, the estimation of the confounder becomes biased, which is substantiated by Fig. \ref{fig_sc}(b).

\textbf{Biased estimation challenge.} The empirical data presented in Fig. \ref{fig_sc}(b) under the imbalanced test set demonstrates the inherent inability of degradation classifiers to achieve perfect estimation. The further performance decline observed in the balanced test set confirms that the existence of spurious correlation amplifies the inaccuracy in estimating $T$.

Both challenges result in the estimation bias of confounders, and then the confounding path is not sufficiently blocked, preventing the AiOIR model from effectively learning the deconfounded causal effects.

\textbf{Causal discovery for correctness validation of SCM graph.} To establish the deconfounding approach on a complete theoretical foundation, we validate the SCM graph that models the two challenges of practical AiOIR based on causal discovery theory. We first employ the Fast Causal Inference (FCI) algorithm~\cite{spirtes2000causation,spirtes2013causal} to learn the Partial Ancestral Graph (PAG) from real-world datasets, where the algorithm outputs $C \leftrightarrow T$. Furthermore, we adopt the Invariant Causal Prediction (ICP) approach~\cite{peters2016causal} to identify the latent confounder as environment $E$. The detailed experimental procedures and results are provided in the Appendix \ref{sec:expri_cd}.

\section{Methodology}
\label{method}
\begin{figure*}[htbp]
\centering
\vspace{-0.5cm}
\includegraphics[width=0.75\linewidth]{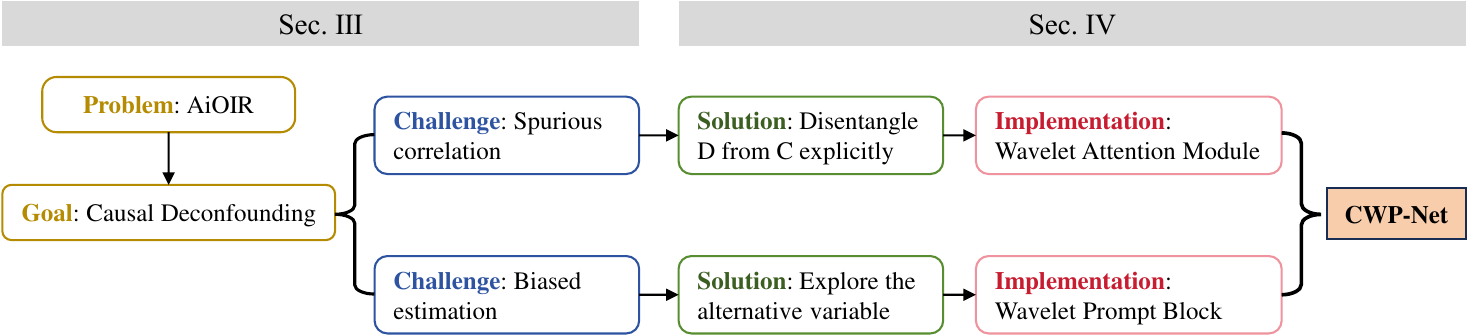}
\vspace{-0.2cm}
\caption{The pipeline of our causal analysis and the corresponding causality-guided method, i.e., CWP-Net.}
\label{fig_connect}
\end{figure*}
In this section, we will begin by introducing our causal deconfounding methods to address the above challenges, as shown in Fig. \ref{fig_connect}. Following this, we will provide a comprehensive explanation of our proposed wavelet attention module, designed to mitigate spurious correlations, and the wavelet prompt block, aimed at generating the alternative variable for biased estimated $T$. Subsequently, we will introduce the overall architecture of CWP-Net. Finally, we will illustrate the loss function for training.

\subsection{Causal Deconfounding via Backdoor Adjustment}
\label{sec:deconfounding}
\textbf{Solution for biased estimation challenge.} Due to the biased estimation challenge, we hypothesize that $T$ is unobservable. Instead, we turn to search for the alternative variable for $T$ to perform effective causal deconfounding, thereby enabling the capture of the true causal effect $\mathbb{P}(Y|do(X))$.

\begin{figure}
\centering
\vspace{-0.2cm}
\includegraphics[width=\linewidth]{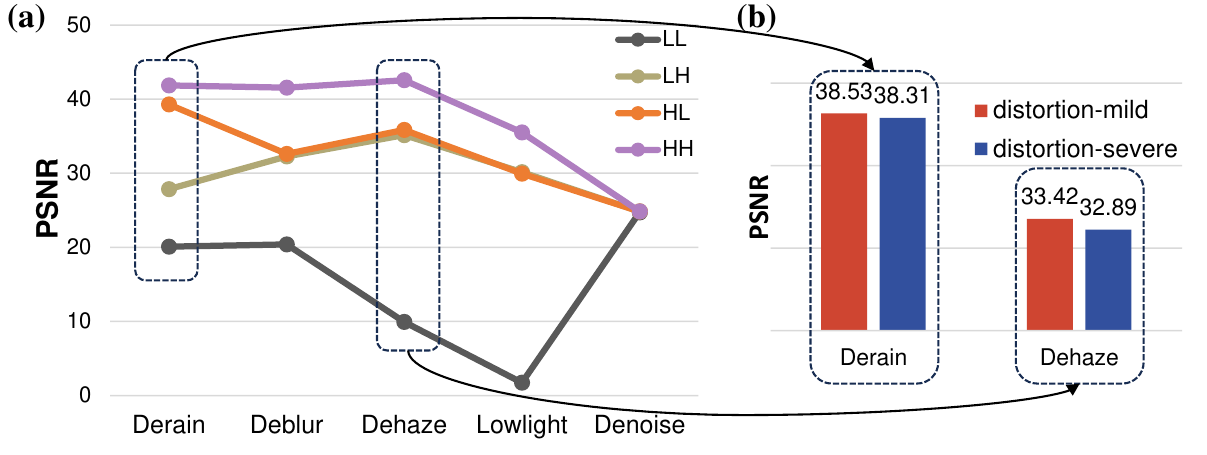}
\vspace{-0.5cm}
\caption{(a) The distortion severity of five degradation patterns on each wavelet subband. The ``PSNR'' is calculated with a higher value indicating milder distortion. Here, we roughly consider items with a PSNR above 30 dB (green dashed line) as distortion-mild subbands, and those below 30 dB as distortion-severe subbands. (b) Experiment results of introducing prompt components to distortion-mild (e.g., ``HH'' and ``HL'' for derain task) and distortion-severe (e.g., ``LH'' and ``LL'' for derain task) subbands of features in AiOIR model.}
\vspace{-0.5cm}
\label{fig:alternative_variable}
\end{figure}
$\mathbf{T \rightarrow P \rightarrow Y.}$ As seen in Fig. \ref{fig_scm}(c), we first introduce a variable $P$ and replace $T \rightarrow Y$ with $ T \rightarrow P \rightarrow Y$ according to the observation that the degradation pattern affects prompted wavelet subbands (i.e., $T \rightarrow P$). This observation can be concluded from the experimental evidence in Fig. \ref{fig:alternative_variable}. Specifically, we first analyze the distortion severity of four wavelet subbands of degraded images in the training set of different $T$ in Fig. \ref{fig:alternative_variable}(a). Besides, the prompt component is introduced to modulate different wavelet subbands of features through two procedures: prompt generation and prompt-feature interaction, whose implementation details can be found in Sect. \ref{sec:wpb}. The experimental results in Fig. \ref{fig:alternative_variable}(b) indicate that prompting on distortion-mild subbands is more effective. However, the distortion-mild subbands for different $T$ have discrepancies. Therefore, the degradation pattern affects the prompted wavelet subbands (i.e., $T \rightarrow P$). 

According to the backdoor criterion \cite{pearl2011graphical}, the variables on the backdoor path $X\leftarrow D \leftarrow T \rightarrow P \rightarrow Y$ from $X$ to $Y$, can be used for deconfounding. 
Due to the unobservability of $T$, the variables available for adjustment are $D$ and $P$. However, adjusting $D$ encounters obstacles. Firstly, the \textit{value space} of $D=d_i$ is too broad, whereas $T$ only encompasses a limited number of degradation patterns. Therefore, the value space of $P$ is easier to traverse compared to $D$, as the values of $P$ are determined by the specific degradation task. Secondly, the \textit{available value set} of variable $D$ in the training set is too small due to the challenge of collecting degraded images with the same content but different degradation features. To instantiate $\mathbb{P}(Y|X,D=d_i)$, we need to perform counterfactual data augmentation to artificially replace degradation features and obtain usable samples, which is complicated to implement. 

Therefore, the variable $P$ is used as the adjustment variable to acquire the causation (We prove the identifiability of $P(Y|do(X))$ for Fig. \ref{fig_scm}(c) in Appendix \ref{sec_identi}):
\begin{align}
\mathbb{P}(Y|do(X))&=\sum_{p_i}\mathbb{P}(Y|X,P=p_i)\mathbb{P}(P=p_i).\label{eq:adj2}
\end{align} 

Let the probability space of prompted wavelet subbands $P$ be denoted as $\mathcal{M}=(\mathcal{P}, \mathcal{E}, \mathbb{P}):$
\begin{equation}
\left\{
\begin{array}{l}
\mathcal{P}=\{ P_{LL}, P_{LH}, P_{HL}, P_{HH} \}, \\
{p}_i \in \mathcal{E}= \left\{ \left ( \omega_{LL}\cdot P_{LL}, \omega_{LH}\cdot P_{LH}, \omega_{HL}\cdot P_{HL},\omega_{HH}\cdot P_{HH}\right ) \right\},
\end{array}
\right.
\end{equation}
where the sample space $\mathcal{P} $ represents wavelet subbands modulated by the prompt components. The event space $\mathcal{E} $ is a weighted combination of all subbands $P_{*} $. The value range of $\omega_{*}, *\in \{LL, LH,HL,HH\} $ is $[0,1] $. Naturally, when $\omega_{*} $ takes values of 0 or 1, we obtain specific combinations, such as $p_i=\left( P_{HL},P_{HH}\right) $ for deraining and $p_i=\left( P_{LH},P_{HL},P_{HH}\right) $ for dehazing. We define the value of random variable $P$ as $p_i $, where $P=p_i $ represents the weighted combination of four prompted wavelet subbands, i.e., $p_i \in \mathcal{E} $.

To instantiate the adjustment formula (i.e., Eq. \eqref{eq:adj2}), we design a wavelet prompt block (WPB) to learn the probability distribution over all possible $p_i $ during each training iteration. The summation term accumulates the contribution of the distribution of the combination $p_i $ of prompted wavelet subbands $P_{*} $ learned by the WPB in each epoch to the restoration network, that is, the degree of contribution to the causal effect $P(Y|do(X))$. To model the distribution of $p_i $, we introduce a degradation-based weight estimator (DWE) to output the probability of selecting a particular $p_i $, which can be seen in Sect. \ref{sec:wpb}. We construct CWP-Net with the wavelet prompt block inserted into the skip-connection of U-Net architecture to model the probability $\mathbb{P}(Y|X,P=p_i) $, which can be seen in Fig. \ref{fig:method} and Sect. \ref{sec:overall_frame}.

\textbf{Solution for spurious correlation challenge.} Due to the spurious correlation challenge, the estimated $\mathbb{P}(P=p_i)$ will have an unexpected correlation with $C$ except for the expected correlation with $D$. To obtain accurate values of $p_i$, we need to block $C \dashleftarrow \dashrightarrow T$, and there exists only one path $ D \leftarrow T \rightarrow P$, ensuring $P \upmodels C$, as shown in Fig. \ref{fig_scm}(c). To remove the spurious correlation, we introduce the wavelet attention module, where $D$ is disentangled explicitly from $C$ and then is utilized to estimate $p_i$, as elaborated on in Sect. \ref{sec:wa} and Sect. \ref{sec:wpb}.

\begin{figure*}
\centering
\vspace{-0.5cm}
\includegraphics[width=0.95\textwidth]{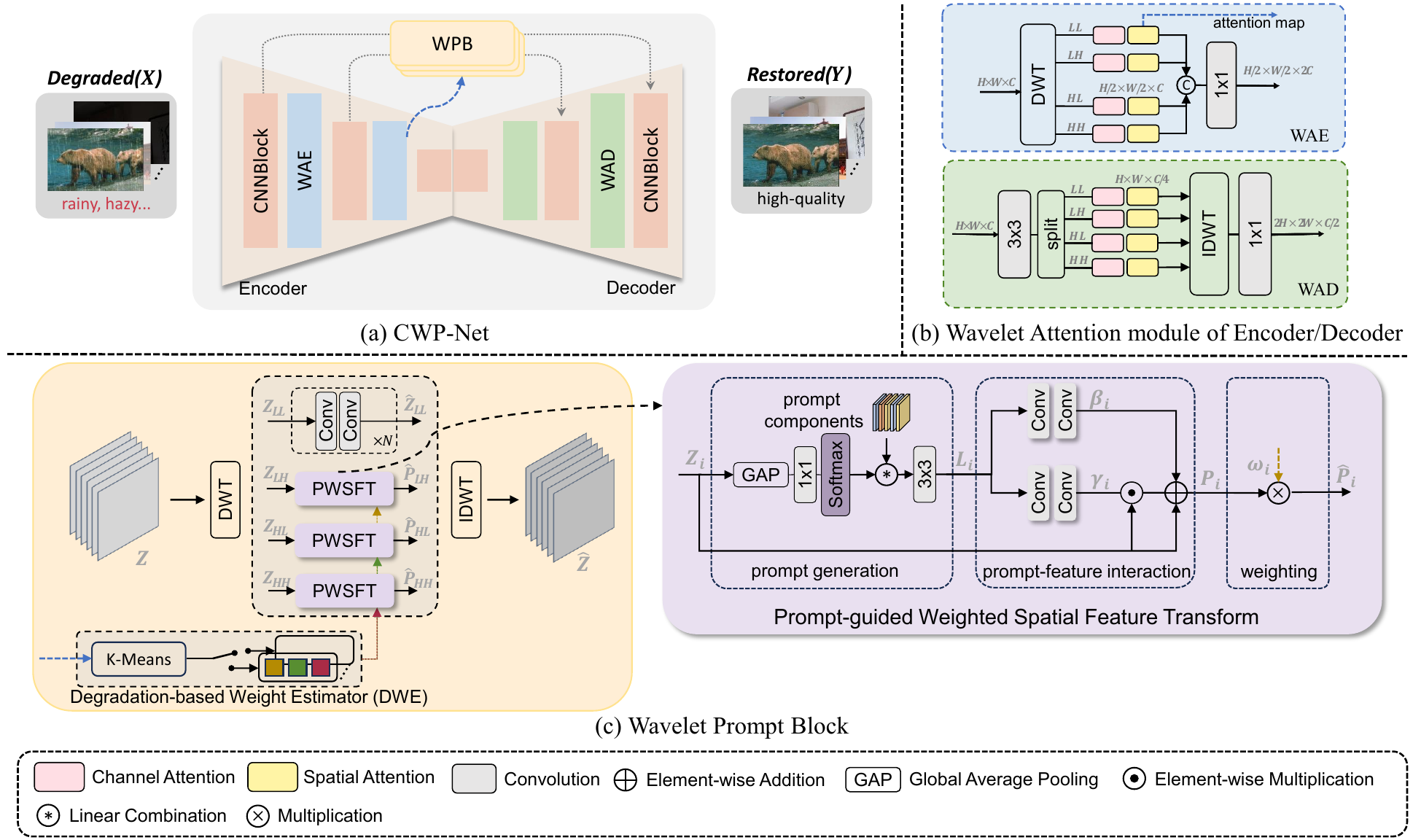}
\vspace{-0.2cm}
\caption{Architecture of our proposed CWP-Net. (a) provides an overview of the U-shaped network for AiOIR, in which WAE and WAD are symmetrically placed at each level of the encoder and decoder, and WPB is inserted into the skip connection. (b) WAE and WAD extract attention maps of different wavelet subbands in both spatial and channel dimensions. (c) WPB generates weighted prompted wavelet subbands for backdoor adjustment based on the low-frequency wavelet attention map. $Z_{*}$ and $P_{*}$ denote the wavelet subbands before and after prompt-based modulation, respectively.}
\vspace{-0.5cm}
\label{fig:method} 
\end{figure*}

\subsection{Wavelet Attention Module for Disentangling Degradation Features from Semantic Features}
\label{sec:wa}

\textbf{Addressing spurious correlation challenge via the wavelet attention module.} To mitigate the spurious correlation, we need to disentangle the degradation representation from the semantic representation explicitly to represent $D$. Then the degradation representation is utilized to obtain the values of $p_i$, thereby eliminating the unexpected correlation between $P$ and $C$. Concretely, considering that the degradation features are more significant in the low-frequency part of images than other frequency parts, as illustrated in Fig. \ref{fig:alternative_variable}(a), the wavelet attention module opts for attention maps of low-frequency features as degradation representations, which focus more on the degraded regions in the image and are semantically independent. We will detail the architecture of the wavelet attention module below.

\textbf{Implementation.} The architecture of wavelet attention module is depicted in Fig. \ref{fig:method}(b). Specifically, we decouple the original features into different frequency subbands through wavelet decomposition and apply attention modules to each wavelet coefficient separately. Only the attention map of low-frequency parts is utilized as the degradation representation. The wavelet attention modules are symmetrically inserted in each scale of both the encoder and decoder. Due to the disparity in feature size between the encoder and decoder (i.e., the encoder downsamples the features and the decoder upsamples the features), we design the wavelet attention module of the encoder and the wavelet attention module of the decoder respectively, which will be explained in detail below.

\begin{figure}[htbp]
\centering
\includegraphics[width=\linewidth,scale=1.00]{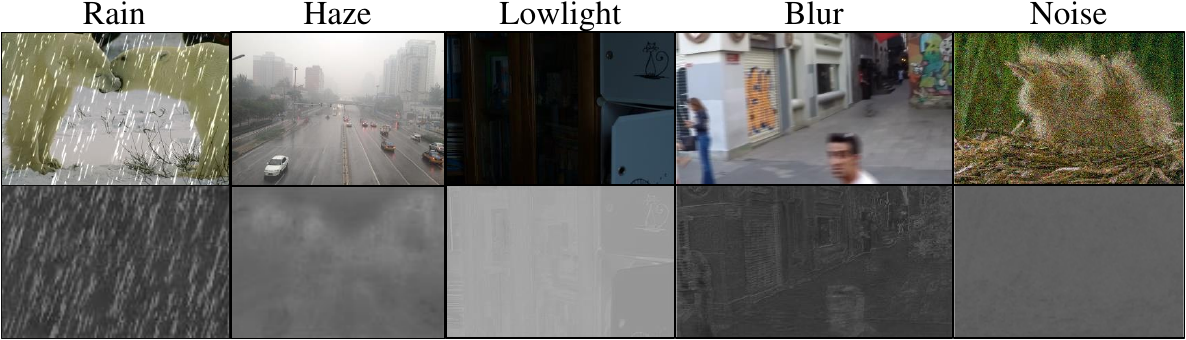}
\vspace{-0.5cm}
\caption{Visualization of the attention maps of the low-frequency subband in WAE. The first row depicts the input degraded image, and the second row illustrates the corresponding attention map. It provides evidence that the map focuses on the degraded regions and decouples the degradation features from the semantic features.}
\vspace{-0.5cm}
\label{fig:attention_map}
\end{figure}

1) \textit{Wavelet attention module of encoder.} WAE first transforms spatial features into frequency domain through discrete wavelet transform $\rm{DWT(\cdot)}$ and obtains four wavelet coefficients.  Inspired by \cite{woo2018cbam}, channel attention and spatial attention are then sequentially applied to each wavelet coefficient to adaptively extract important information on the channel and spatial dimension, respectively. These coefficients are then concatenated together and reduced by a $1\times1$ convolution in the channel dimension. Since $\rm{DWT(\cdot)}$ reduces the feature resolution to half of the original, we omit the downsampling layer and replace it with WAE. The overall process is expressed as follows:
\begin{equation}
\begin{aligned}
    &F_{LL}^e,F_{LH}^e,F_{HL}^e,F_{HH}^e=\mathrm{DWT}(F^e), \\
    &\hat{F}_{j}^e=SA_j(CA_j(F_j^{e})),j\in\{LL,LH,HL,HH\},\\  &\hat{F}^{e}=Conv1(concat(\hat{F}_{LL}^e,\hat{F}_{LH}^e,\hat{F}_{HL}^e,\hat{F}_{HH}^e)),
\end{aligned}
\end{equation}
where $Conv1$ represents the $1\times1$ convolution layer, $F^e$ and $\hat{F}^{e}$ represent the input and output features of WAE, $CA_{j}$ and $SA_{j}$ represent the channel attention and spatial attention corresponding to the $j$-th wavelet coefficient, respectively.

2) \textit{Wavelet attention module of decoder.} As WAE has transformed the features into the frequency domain, we consider using a $3\times3$ convolution to extract high- and low-frequency features in the decoder, as illustrated in Fig. \ref{fig:method}(b). Similar to WAE, we apply channel attention and spatial attention on each wavelet feature to refine the features. The refined wavelet features are then fed into $\rm{IDWT(\cdot)}$ to obtain the spatial domain representation, followed by a $1\times1$ convolution to increase the channel dimension. Similarly, $\rm{IDWT(\cdot)}$ increases the spatial resolution by twice the original, so the original upsampling layer is substituted by WAD. Overall, the process of WAD is represented as follows:
\begin{equation}
\begin{aligned}
    &F_{LL}^d,F_{LH}^d,F_{HL}^d,F_{HH}^d=Conv3(F^d),\\
    &\hat{F}_{j}^d=SA_j(CA_j(F_{j}^d)),j\in\{LL,LH,HL,HH\},\\
    &\hat{F}^{d}=Conv1(\mathrm{IDWT}(\hat{F}_{LL}^d,\hat{F}_{LH}^d,\hat{F}_{HL}^d,\hat{F}_{HH}^d)),
\end{aligned}
\end{equation}
where $Conv3$ represents the $3\times3$ convolution layer, $F^d$ and $\hat{F}^{d}$ represent the input and output features of WAD, respectively.

\textit{Channel attention and spatial attention.}  For input features $\mathbf{F} \in \mathbb{R}^{H\times W\times C }$, the channel attention module learns the attention vectors $\mathbf{M}_c \in \mathbb{R}^{1\times1\times C}$ which represent the importance of different channels and are multiplied by input features to obtain channel-updated features. Similarly, the spatial attention module learns the attention maps $\mathbf{M}_s \in \mathbb{R}^{H\times W\times1}$ multiplied by input features to obtain spatial-updated features.

Additionally, the spatial attention map of the low-frequency part $F_{LL}^e$ in WAE of the second scale is utilized as the degradation representation. Visualization examples of the attention map are shown in Fig. \ref{fig:attention_map}, which can demonstrate its focus solely on the degraded features.

\subsection{Wavelet Prompt Block for Exploring Alternative Variable}
\label{sec:wpb}

\textbf{Addressing biased estimation challenge via the wavelet prompt block.} The proposed wavelet prompt block is designed to estimate the probability $\mathbb{P}(P=p_i)$ of the alternative variable. Concretely, the implications of $p_i$ can be represented by:
\begin{equation}
    \begin{aligned}
        % &\mathcal{P}=\left \{ P_{LL},P_{LH},P_{HL},P_{HH} \right \},\\
        % &\Omega =\left \{ (\omega_{LL}, \omega_{LH},\omega_{HL},\omega_{HH}) | \omega_{*}\in\left [ 0,1 \right ]  \right \}, \\
        p_i&=\left ( \hat{P}_{LL}, \hat{P}_{LH}, \hat{P}_{HL} ,\hat{P}_{HH}\right )\\
        &=\left ( \omega_{LL}\cdot P_{LL}, \omega_{LH}\cdot P_{LH}, \omega_{HL}\cdot P_{HL},\omega_{HH}\cdot P_{HH}\right ),
    \end{aligned}
\end{equation}
where $P_{*}$ represents features' wavelet subbands modulated by the prompt components, $\omega_{*}$ denotes the weights of specific subbands, and $\hat{P}_{*}$ indicates weighted prompted wavelet subbands. Our WPB primarily comprises two modules: a degradation-based weight estimator (DWE) and a prompt-guided weighted spatial feature transform (PWSFT) module. The DWE module is employed to determine the value of the weights $\omega_*$ based on the degradation representations, avoiding the need for expert knowledge to determine whether $\omega_*$ is 0 or 1. Then the weights are further used by the PWSFT module to obtain $\hat{P}_{*}$.  In the following, we present the detailed architectures of PWSFT and DWE, respectively.

\textbf{Implementation.} The architecture of wavelet prompt block is illustrated in Fig. \ref{fig:method}(c). Specifically, $\rm{DWT(\cdot)}$ is firstly applied on the encoder features $Z$ to obtain four coefficients.  As shown in Fig. \ref{fig:alternative_variable}(a), each degradation pattern exhibits severe degradation in the low-frequency subband. Accordingly, we incorporate this prior knowledge by utilizing $N$ plain convolution layers to refine the low-frequency subband $Z_{LL}$ (i.e., set $\omega_{LL}=0$). For the other three high-frequency subbands $Z_{*}$, the corresponding weighted prompted wavelet subbands $\hat{P}_{*}$ are generated through the DWE module and PWSFT modules. Finally, the refined $\hat{Z}_{LL}$ and $\hat{P}_{*}$ are transformed back to the spatial domain $\hat{Z}$ using $\rm{IDWT(\cdot)}$, which is described as follows:
\begin{equation}
    \hat{Z}=\mathrm{IDWT}(\hat{Z}_{LL}, \hat{P}_{LH}, \hat{P}_{HL}, \hat{P}_{HH}).
\end{equation}

Next, we detail the specific implementation of PWSFT and DWE, respectively.

1) \textit{Prompt-guided weighted spatial feature transform.} Overall, the PWSFT module obtains weighted prompted wavelet subbands using the following \textit{weighting} procedure:
\begin{equation}
    \hat{P}_{j}=\omega _{j}*P_{j},j\in\{LH,HL,HH\},
\end{equation}
where $\hat{P}_{j}$ represents weighted prompted wavelet subbands, $\omega_{j}$ denotes the weights derived by DWE, and $P_{j}$ is obtained by modulating the wavelet features $Z_{j}$ through \textit{prompt generation} and \textit{prompt-feature interaction} procedures.

Specifically, for each wavelet subband, multiple learnable prompts are maintained to serve as task-aware priors. Inspired by \cite{potlapalli2024promptir}, these prompts are dynamically weighted and combined in an input-conditioned manner to generate a unified prompt. Subsequently, we employ this unified prompt to interact with and modulate the corresponding wavelet subbands. To encourage spatially adaptive interactions that facilitate pixel-level restoration, we adopt the spatial feature transform, which generates pixel-wise scaling and shifting parameters to facilitate prompt-feature interaction. We next present the formalized operations of the above procedure.

\begin{table*}[!ht]
\vspace{-0.5cm}
\caption{Quantitative results of different methods under the five-pattern setting. PSNR($\uparrow$) and SSIM($\uparrow$) are reported. The best and second-best results are marked in \textbf{bold} and \underline{underlined}, respectively.}
\vspace{-0.2cm}
\label{tab:contrast_five}
\centering
\begin{tabular}{l|cc|cc|cccccc|cc}
\toprule 
\multirow{3}{*}{Method} & \multicolumn{2}{c|}{\multirow{2}{*}{\makecell[c]{Dehazing} }} & \multicolumn{2}{c|}{\multirow{2}{*}{\makecell[c]{Deraining}}} & \multicolumn{6}{c|}{Denoising} & \multicolumn{2}{c}{\multirow{2}{*}{Average}}\\
 & & & & &\multicolumn{2}{c}{$\sigma=15$} & \multicolumn{2}{c}{$\sigma=25$} & \multicolumn{2}{c|}{$\sigma=50$} & & \\
 & PSNR & SSIM & PSNR & SSIM & PSNR & SSIM & PSNR & SSIM & PSNR & SSIM & PSNR & SSIM \\
\midrule
BRDNet \cite{tian2020image} & 23.23 & 0.895& 27.42 & 0.895& 32.26 & 0.898& 29.76 & 0.836& 26.34 & 0.836& 27.80 & 0.843\\
 LPNet \cite{gao2019dynamic} & 20.84 & 0.828& 24.88 & 0.784& 26.47 & 0.778& 24.77 & 0.748& 21.26 & 0.552&23.64 & 0.738\\
 FDGAN \cite{dong2020fd} & 24.71 & 0.929& 29.89 & 0.933& 30.25 & 0.910& 28.81 & 0.868& 26.43 & 0.776&28.02 & 0.883\\
 MPRNet \cite{zamir2021multi} & 25.28 & 0.955& 33.57 & 0.954& 33.54 & 0.927& 30.89 & 0.880& 27.56 & 0.779&30.17 & 0.899\\
 \midrule
 DL \cite{fan2019general} & 26.92 & 0.931& 32.62 & 0.931& 33.05 & 0.914& 30.41 & 0.861& 26.90 & 0.740&29.98 & 0.875\\
 AirNet \cite{li2022all} & 27.94 & 0.962& 34.90 & 0.968& 33.92 & 0.933& 31.26 & 0.888& 28.00 & 0.797&31.20 & 0.910\\
 PromptIR \cite{potlapalli2024promptir} & 30.58 & 0.974& 36.37 & 0.972& 33.98 & 0.933& 31.31 & 0.888& 28.06 & 0.799&32.06 & 0.913\\
 Lin et al. \cite{lin2023improving} & \underline{31.63} & \underline{0.980}& \underline{37.58} & \underline{0.979}& \underline{34.01} & \underline{0.933}& \underline{31.39} & \underline{0.890}& \underline{28.18} & \underline{0.802}& \underline{32.56} & \underline{0.916}\\
 \midrule
\rowcolor{orange}
\textbf{CWP-Net}& \textbf{33.21} & \textbf{0.983} & \textbf{38.68} & \textbf{0.984} & \textbf{34.14} & \textbf{0.936} & \textbf{31.48} & \textbf{0.893}  & \textbf{28.22} & \textbf{0.805} & \textbf{33.15} & \textbf{0.920} \\
 \textit{Gains}& +1.58& +0.003& +1.10& +0.005& +0.13& +0.003& +0.09& +0.003& +0.04& +0.003& +0.59&+0.004\\
  \bottomrule
\end{tabular}
\vspace{-0.3cm}
\end{table*}

\textit{Prompt generation.} Following \cite{potlapalli2024promptir}, we first generate the input-conditioned prompt by a combination of $M$ learnable prompt components, which can be expressed as follows:
\begin{equation}
\begin{aligned}
     &L_{j} = Conv3_{j}(\sum_{c=1}^{M} \alpha_{j,c} *L_{j,c}), j\in \{LH,HL,HH\},\\
     &\alpha_{j} = Softmax(Conv1_{j}(\mathrm{GAP}(Z_{j}))),
\end{aligned}
\end{equation}
where $L_{j,c}\in\mathbb{R} ^{H\times W \times C}$ denotes learnable prompt components, $\alpha_{j}\in\mathbb{R} ^{1\times 1 \times M}$  denotes the weight of each learnable prompt component, $\rm{GAP(\cdot)}$ denotes global average pooling operator, and $L_{j}$ denotes input-conditioned prompt.

\textit{Prompt-feature interaction}. The prompt is expected to modulate input features at the pixel level, so spatial feature transform (SFT) \cite{wang2018recovering} is utilized to achieve prompt-feature interaction. SFT learns a pair of affine transformation parameters (i.e., scaling parameters $\gamma_{j}$ and shifting parameters $\beta_{j}$) conditioned on prompts and uses these parameters to modulate the original features. Mathematically:
\begin{equation}
\begin{aligned}
    &\beta_{j}=Conv1_{j}(Conv1_{j}(L_{j})),\\
    &\gamma_{j}=Conv1_{j}(Conv1_{j}(L_{j})),\\
    &P_{j} = \gamma _{j} \odot  Z_{j} + \beta _{j} + Z_{j},j\in\{LH,HL,HH\}.
\end{aligned}
\end{equation}

2) \textit{Degradation-based weight estimator.} As outlined in Fig. \ref{fig:alternative_variable}, 
different degradation patterns exhibit discrepancies in distortion-mild wavelet subbands. Therefore, the weights of three prompted high-frequency wavelet subbands are determined by the specific degradation pattern. To be more specific, for the $K$ degradation patterns within the training set, we define the learnable parameter $W\in\mathbb{R} ^{3\times K}$, which is optimized through backpropagation with values ranging from 0 to 1. $W_{j,k}$ represents the weight of the $j$-th high-frequency subband corresponding to the $k$-th degradation pattern. To determine the degradation pattern corresponding to the current degraded image, we employ K-Means clustering \cite{jain1988algorithms} on the degradation representations obtained by the WAE module. Then the weight $\omega_{j}$ is chosen based on the category to which the current sample belongs.

The DWE is optimized through a delayed update strategy to ensure the robustness of the degradation representation. To be more specific, $W$ is initialized to a matrix with all values to be one and updated after $N_{w}$ epochs.

\subsection{Overall Framework}
\label{sec:overall_frame}
The overall architecture of CWP-Net is illustrated in Fig. \ref{fig:method}(a). We adopt the widely used U-shaped encoder-decoder architecture for efficient AiOIR. Specifically, the input images $\mathbf{X}\in \mathbb{R}^{H\times W\times3} $ with different degradations are passed through a $3\times3$ convolutional layer to obtain shallow embeddings. Subsequently, the embeddings are then fed into the multi-scale encoder to obtain the hierarchical feature representations. Each scale of encoder comprises a CNNBlock \cite{cui2023irnext} and a wavelet attention module of encoder (WAE) to perform feature transformation. Symmetrically, each scale of the decoder consists of a wavelet attention module of decoder (WAD) and a CNNBlock to progressively restore the image to its original resolution and finally obtain the restored image $\mathbf{Y}\in \mathbb{R}^{H\times W\times3}$. Given that the encoder features contain more task-specific information, our proposed wavelet prompt block (WPB) is integrated into the skip-connection to estimate  $\mathbb{P}(Y|X,P=p_i)$.
Additionally, we also incorporate the multi-input multi-output strategy in the architecture to facilitate progressive learning following \cite{cho2021rethinking,cui2023selective,tu2022maxim}.

\subsection{Loss Functions}
Our proposed CWP-Net is optimized by the combination of reconstruction loss and frequency loss.

\textbf{Reconstruction loss.} We use the $\mathcal{L}_1$  distance between restored image $Y$ and ground truth $\hat{Y}$ to measure the reconstruction performance, which is defined by:
\begin{equation}
    \mathcal{L} _{rec} =\sum_{k=1}^{3}  \left \| Y_k-\hat{Y_k} \right \| _1,
\end{equation}
where $k$ represents the index of multi-scale output.

\textbf{Frequency loss.} To enhance the recovery of high-frequency details, we apply a fast Fourier transform to the restored image and the ground truth, respectively, and obtain $ \mathcal{F}(\hat{Y})$ and $\mathcal{F}(Y) $. The $\mathcal{L}_1$ distance between $\mathcal{F}(\hat{Y})$ and $\mathcal{F}(Y)$ is employed as the frequency loss, which can be expressed as follows:
\begin{equation}
    \mathcal{L} _{fre} =\sum_{k=1}^{3}  \left \| \mathcal{F}(Y_k)-\mathcal{F}(\hat{Y_k}) \right \| _1,
\end{equation}
where $k$ represents the index of multi-scale output.

\textbf{Total loss.} The overall loss function for training CWP-Net is the combination of reconstruction loss and frequency loss:
\begin{equation}
    \mathcal{L} = \mathcal{L} _{rec} + \lambda \mathcal{L} _{fre},
\end{equation}
where $\lambda$ represents the balancing weight, which is empirically set to 0.1.

\begin{table*}[!ht]
\caption{Quantitative results of different methods under the seven-pattern setting. PSNR($\uparrow$) and SSIM($\uparrow$) are reported. The best and second-best results are marked in \textbf{bold} and \underline{underlined}, respectively.}
\vspace{-0.2cm}
\label{tab:contrast_seven}
\centering
\begin{tabular}{l|cc|cc|cc|cc|cc|cc}%七列
\toprule 
\multirow{2}{*}{Method} & \multicolumn{2}{c|}{\makecell[c]{Deraining} } & \multicolumn{2}{c|}{\makecell[c]{Dehazing}} & \multicolumn{2}{c|}{\makecell[c]{Low-Light}} & \multicolumn{2}{c|}{\makecell[c]{Deblurrring}} & \multicolumn{2}{c|}{\makecell[c]{Denoising}} & \multicolumn{2}{c}{\makecell[c]{Average}}\\
 & PSNR & SSIM & PSNR & SSIM & PSNR & SSIM & PSNR & SSIM & PSNR & SSIM & PSNR & SSIM \\
\midrule
NAFNet \cite{chen2022simple} & 35.56 & 0.967 & 25.23 & 0.939 & 20.49 & 0.809 & 26.53 & 0.808 & 31.02 & 0.883 & 27.76 & 0.886 \\
HINet \cite{chen2021hinet} &  35.67 & 0.969&  24.74 & 0.937&  19.47 & 0.800&  26.12 & 0.788&  31.00 & 0.881&  27.40 & 0.875\\
MPRNet \cite{zamir2021multi} &  \underline{38.16} & \underline{0.981}&  24.27 & 0.937&  20.84 & 0.824&  26.87 & 0.823&  31.35 & \textbf{0.889}&  28.27 & 0.890\\
DGUNet \cite{mou2022deep} &  36.62 & 0.971&  24.78 & 0.940&  \underline{21.87} & 0.823&  27.25 & 0.837&  31.10 & 0.883&  28.32 & 0.891\\
MIRNetv2 \cite{zamir2022learning} &  33.89 & 0.954&  24.03 & 0.927&  21.52 & 0.815&  26.30 & 0.799&  30.97 & 0.881&  27.34 & 0.875\\
SwinIR \cite{liang2021swinir} &  30.78 & 0.923&  21.50 & 0.891&  17.81 & 0.723&  24.52 & 0.773&  30.59 & 0.868&  25.04 & 0.835\\
Restormer \cite{zamir2022restormer} &  34.81 & 0.962&  24.09 & 0.927&  20.41 & 0.806&  27.22 & 0.829&  31.49 & 0.884&  27.60 & 0.881\\
\midrule
DL \cite{fan2019general} &  21.96 & 0.762&  20.54 & 0.826&  19.83 & 0.712&  19.86 & 0.672&  23.09 & 0.745&  21.05 & 0.743\\
Transweather \cite{valanarasu2022transweather} &  29.43 & 0.905&  21.32 & 0.885&  21.21 & 0.792&  25.12 & 0.757&  29.00 & 0.841&  25.22 & 0.836\\
TAPE \cite{liu2022tape}&  29.67 & 0.904&  22.16 & 0.861&  18.97 & 0.621&  24.47 & 0.763&  30.18 & 0.855&  25.09 & 0.801\\
AirNet \cite{li2022all} &  32.98 & 0.951&  21.04 & 0.884&  18.18 & 0.735&  24.35 & 0.781&  30.91 & 0.882&  25.49 & 0.846\\
IDR  \cite{zhang2023ingredient} &  35.63 & 0.965&  \underline{25.24} & \underline{0.943}&  21.34 & \underline{0.826}&  \underline{27.87} & \textbf{0.846}&  \textbf{31.60} & \underline{0.887}&  \underline{28.34} & \underline{0.893}\\
\midrule
\rowcolor{orange}    
\textbf{CWP-Net} &  \textbf{38.52}&  \textbf{0.983}&  \textbf{32.88}&  \textbf{0.981}&  \textbf{21.92}& \textbf{0.844}& \textbf{27.94}& \underline{0.845}& \underline{31.54}& 0.877& \textbf{30.56}& \textbf{0.906}\\
 \textit{Gains}& +0.36& +0.002& +7.64& +0.038& +0.05& +0.018& +0.07& -0.001& -0.06& -0.012& +2.22&+0.013\\
  \bottomrule
\end{tabular}
\end{table*}

\section{Experiments} 
Following the prior works \cite{potlapalli2024promptir,zhang2023ingredient}, we evaluate the proposed method on five image restoration tasks (i.e., image dehazing, deraining, denoising, deblurring, and low-light image enhancement) under two all-in-one settings: (a) five-pattern setting (denoising of $sigma$ in $\{15,25,50\}$, deraining, dehazing); (b) seven-pattern setting (denoising of $sigma$ in $\{15,25,50\}$, deraining, dehazing, deblurring, and low-light image enhancement).

\subsection{Experimental Setup}
\label{sec:expri_set}
\textbf{Implementation details.} 
To implement our CWP-Net, the number of plain convolutional layers $N$ is set to 4. For training, the total batch size is set to 48. We use Adam optimizer with the initial learning rate as 2e-4, which is gradually reduced to 1e-6 using the cosine annealing strategy. The training images are cropped to $224\times224$ patches and then augmented by random flipping and rotation. The overall model undergoes training for a total of 150 epochs, with the delayed update epoch $N_w$ set to 100. All experiments are conducted on RTX 3090 GPUs.

\textbf{Datasets.}
\label{sec:dataset}
Following the current method \cite{potlapalli2024promptir}, we prepare datasets for different tasks. For image denoising, a training dataset comprising 5,144 images is constructed by combining the WED dataset \cite{ma2016waterloo} and BSD400 dataset \cite{arbelaez2010contour}. To generate noisy counterparts, random Gaussian noise with three different degrees ($\sigma\in\{15,25,50\}$) is added to the clean images. The BSD68 dataset \cite{martin2001database} is used for testing. 

For the deraining task, the Rain100L dataset \cite{yang2020learning} containing 200 training images and 100 testing images is employed. In terms of dehazing task, the outdoor subset of SOTS dataset \cite{li2018benchmarking}, consisting of 72,135 images for training and 500 images for testing, is adopted.

The LOL-v1 dataset \cite{wei2018deep}, which consists of 485 training images and 15 testing images, is utilized for low-light image enhancement. The GoPro dataset \cite{Nah_2017_CVPR} is utilized for deblurring, comprising 2,103 training samples and 1,111 test samples. In five-pattern and seven-pattern AiOIR settings, datasets from the corresponding tasks are mixed together to perform training.

\textbf{Evaluation metrics.}
The commonly used metrics Peak Signal to Noise Ratio (PSNR) \cite{huynh2008scope} and Structural Similarity (SSIM) \cite{wang2004image} are adopted to evaluate the performance of AiOIR methods.

\begin{figure*}
\centering
\includegraphics[width=1\textwidth]{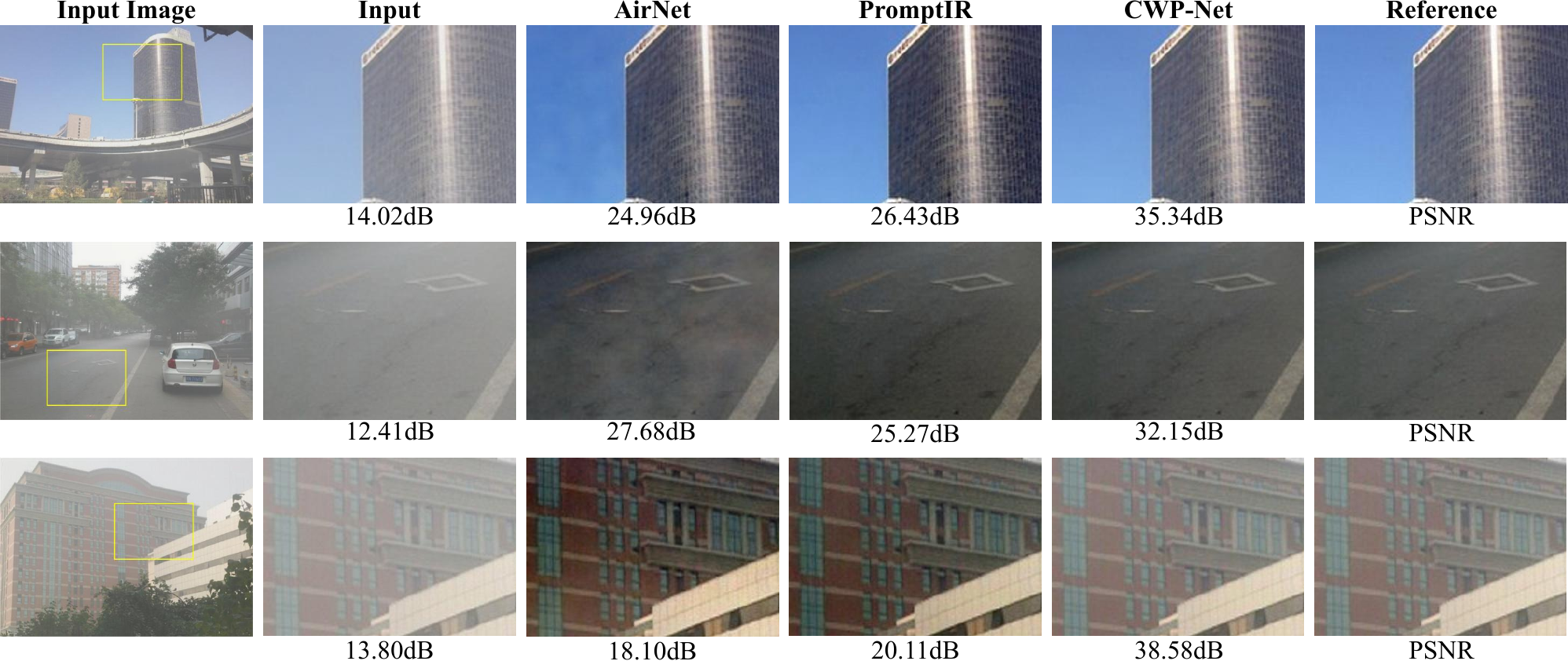}
\vspace{-0.5cm}
\caption{Visual comparisons of all-in-one methods on image dehazing. PSNR of the input image and images recovered by different approaches are listed below the corresponding samples. Our method restores building colors and structural textures that are closest to the reference. Zoom in for the best view.}
\vspace{-0.3cm}
\label{fig:compare_haze}
\end{figure*}

\begin{figure*}
\centering
\vspace{-0.5cm}
\includegraphics[width=1\textwidth]{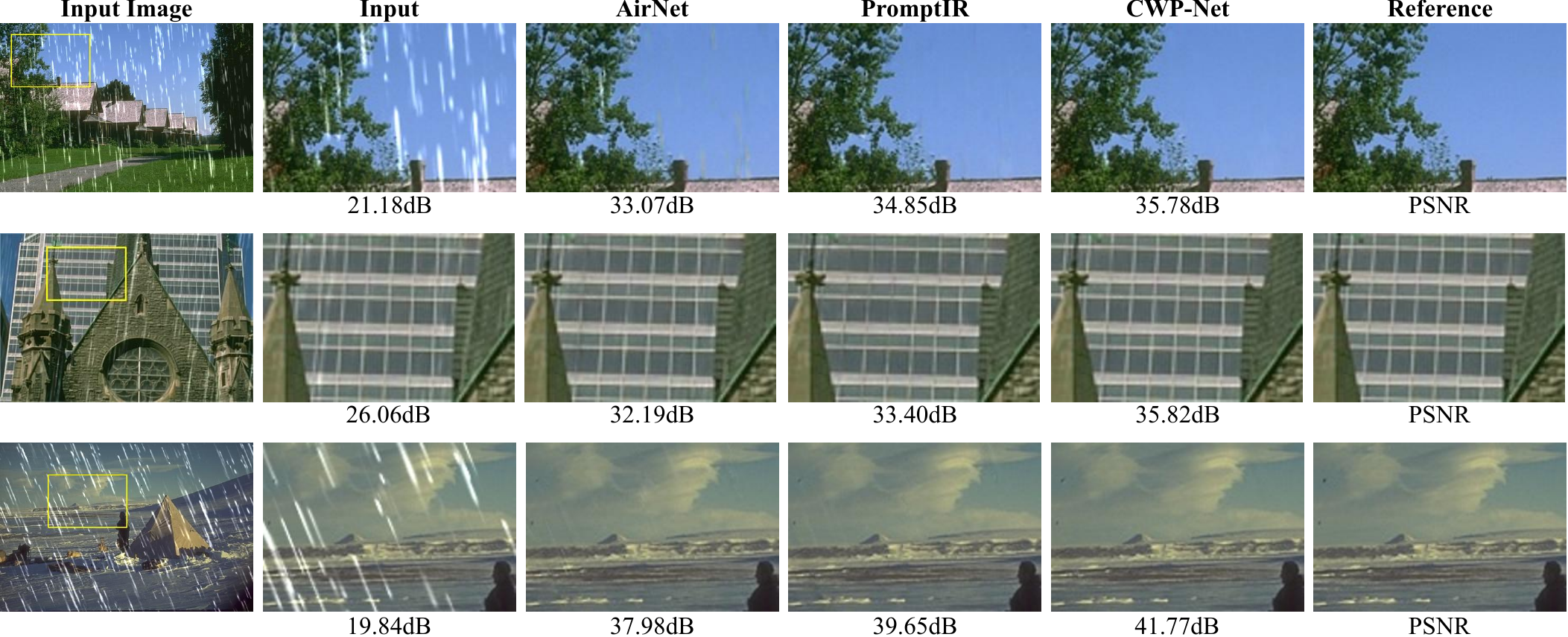}
\vspace{-0.6cm}
\caption{Visual comparisons of all-in-one methods on image deraining. PSNR of the input image and images recovered by different approaches are listed below the corresponding samples. The restored results of our method exhibit fewer rain streaks and a cleaner background. Zoom in for the best view.}
\label{fig:compare_rain}
\end{figure*}

\begin{figure*}
\centering
\includegraphics[width=1\textwidth]{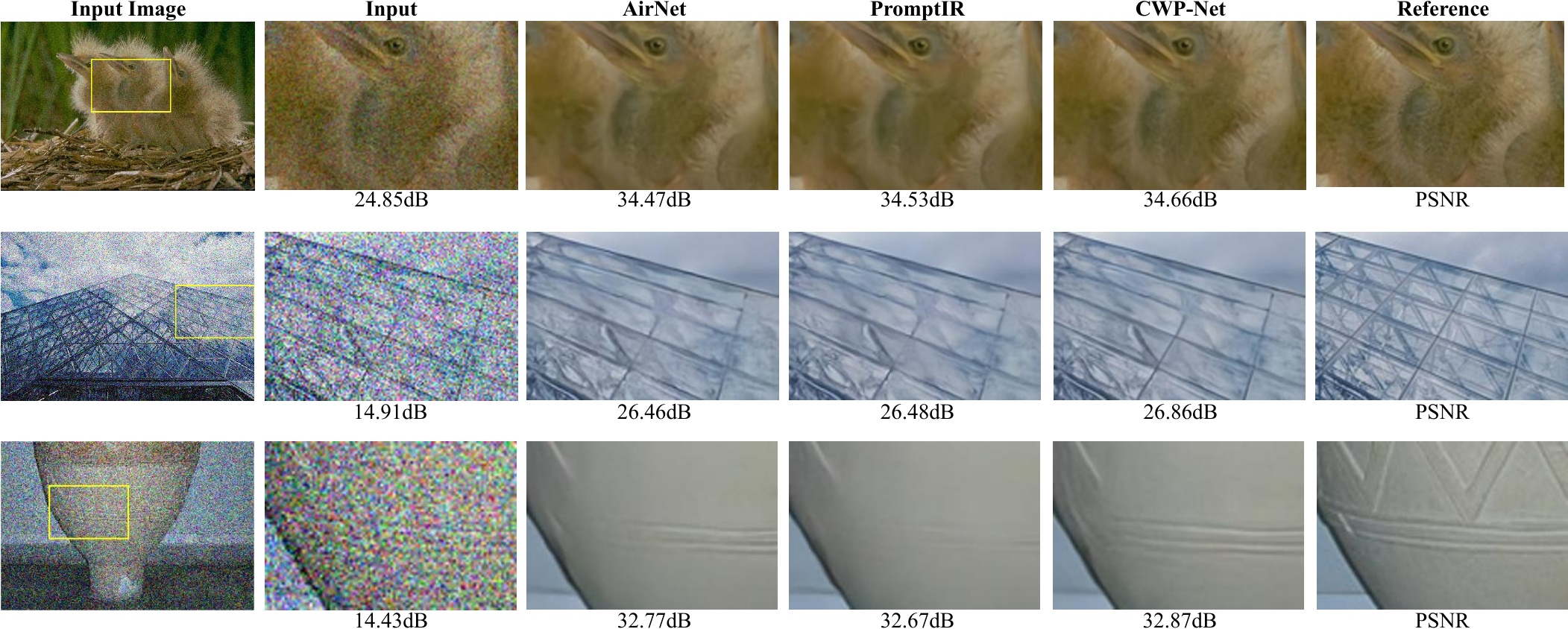}
\vspace{-0.6cm}
\caption{Visual comparisons of all-in-one methods on image denoising. PSNR of the input image and images recovered by different approaches are listed below the corresponding samples. Our method is able to maintain fine detail while removing noise, rather than producing overly smooth results. Zoom in for the best view.}
\vspace{-0.3cm}
\label{fig:compare_noise}
\end{figure*}

\subsection{Comparisons with State-Of-the-Art Methods: five-pattern setting}

To demonstrate the effectiveness of the proposed CWP-Net, we conduct comparisons with four general image restoration models (i.e., BRDNet \cite{tian2020image}, LPNet \cite{gao2019dynamic}, FDGAN \cite{dong2020fd}, MPRNet \cite{zamir2021multi}) and four specialized all-in-one image restoration models (i.e., DL \cite{fan2019general},  AirNet \cite{li2022all}, PromptIR \cite{potlapalli2024promptir}, Lin et al. \cite{lin2023improving}). All of these methods are trained under the all-in-one setting. 

Table \ref{tab:contrast_five} presents the quantitative comparison results. Our proposed method outperforms the previous best method. Specifically, CWP-Net yields a performance gain of 1.58dB in the image dehazing task, and 1.1dB in the image deraining task over Lin et al. Compared to previous methods, CWP-Net obtains superior performance among three noise levels and provides 0.13dB PSNR boost in the noise level of $\sigma=15$. It is worth noting that Lin et al. utilize large-scale pre-trained models, CLIP \cite{radford2021learning}, and Stable Diffusion \cite{rombach2022high} to facilitate the restoration process. In contrast, our method solely relies on the training data mentioned in Sect. \ref{sec:expri_set} and achieves a 0.59dB PSNR improvement over Lin et al. when averaged across five degradation patterns. Compared to PromptIR which prompts the features in the spatial domain with all frequency subbands, our method is more effective and achieves a 1.09dB performance gain on average. AirNet learns the representation of degradation through contrastive learning to guide the restoration process. However, due to the lack of robustness in the degradation representation, it achieves suboptimal performance, particularly in the dehazing task. Our method surpasses AirNet by 1.95dB PSNR on average.

We also provide visual examples for qualitative comparisons from Fig. \ref{fig:compare_haze} to Fig. \ref{fig:compare_noise}. It is evident that our method effectively removes three types of degradation from input images and restores realistic colors and fine textures that closely resemble the ground truth. Instead, AirNet and PromptIR are unable to effectively remove rain streaks from images and perform poorly in recovering local details destroyed by noise. Additionally, these methods also introduce unrealistic colors when removing haze.

\subsection{Comparisons with State-of-the-Art Methods: seven-pattern setting}
We further evaluate our proposed method in a seven-pattern setting, compared with several general image restoration approaches (i.e., NAFNet \cite{chen2022simple}, HINet \cite{chen2021hinet}, MPRNet \cite{zamir2021multi}, DGUNet \cite{mou2022deep}, MIRNetv2 \cite{zamir2022learning}, SwinIR \cite{liang2021swinir}, Restormer \cite{zamir2022restormer}) and all-in-one image restoration approaches (i.e., DL \cite{fan2019general}, Transweather \cite{valanarasu2022transweather}, TAPE \cite{liu2022tape}, AirNet \cite{li2022all}, IDR \cite{zhang2023ingredient}). The results shown in Table \ref{tab:contrast_seven} demonstrate that our CWP-Net surpasses both general restoration methods and all-in-one restoration methods in terms of PSNR and SSIM under seven-pattern setting, showcasing the robustness of our method. When averaged across seven degradation patterns, our method achieves a 2.22dB performance gain over the second-best all-in-one method IDR \cite{zhang2023ingredient} and a 5.07dB performance gain over AirNet \cite{li2022all}.

\begin{table*}[!ht]
\vspace{-0.3cm}
\caption{Additional quantitative results on the balanced test set under the three-pattern setting. PSNR and SSIM are reported.}
\vspace{-0.2cm}
\label{tab:contrast_unbiased}
\centering
\begin{tabular}{l|cc|cc|cccccc|cc}%
\toprule 
\multirow{3}{*}{Method} & \multicolumn{2}{c|}{\multirow{2}{*}{\makecell[c]{Dehazing} }} & \multicolumn{2}{c|}{\multirow{2}{*}{\makecell[c]{Deraining}}} & \multicolumn{6}{c|}{Denoising} & \multicolumn{2}{c}{\multirow{2}{*}{Average}}\\
 & & & & &\multicolumn{2}{c}{$\sigma=15$} & \multicolumn{2}{c}{$\sigma=25$} & \multicolumn{2}{c|}{$\sigma=50$} & & \\
 & PSNR & SSIM & PSNR & SSIM & PSNR & SSIM & PSNR & SSIM & PSNR & SSIM & PSNR & SSIM \\
\midrule
AirNet & 19.46& 0.892& 35.29& 0.968& 36.16& 0.953& 33.69& 0.925& 30.32& 0.863& 30.98& 0.920\\
PromptIR & 19.32& 0.883& 37.35& 0.978& 36.18& 0.953& 33.72& 0.925& 30.35& 0.864& 31.38 & 0.921 \\
\rowcolor{orange}
\textbf{CWP-Net} & \textbf{21.34}& \textbf{0.916}& \textbf{38.47}& \textbf{0.984}& \textbf{36.34}& \textbf{0.955}& \textbf{33.91}& \textbf{0.928}& \textbf{30.58}&\textbf{ 0.869}& \textbf{32.13} &\textbf{ 0.931} \\

\bottomrule
\end{tabular}
\end{table*}

\begin{table*}[!ht]
\caption{Ablation studies of the proposed individual modules on overall performance. PSNR of different models is reported.}
\vspace{-0.3cm}
\label{tab:aba_individual}
\centering
\begin{tabular}{ccccccccccccc}%
\toprule 
Model & Baseline & WAE & WAD & WPB & Derain & Dehaze & Deblur & Low-Light & Denoise & Average & Param(M) & FLOPs(G)\\
\midrule
$(a)$ & \ding{51} & \ding{55} & \ding{55} & \ding{55} & 35.03 & 30.19 & 25.87 & 21.73 & 30.77 & 28.72 & 14.75& 128.93\\
$(b)$ & \ding{51} & \ding{51} & \ding{55} & \ding{55} & 36.52 & 31.02 & 26.02 & 21.64 & 30.92 & 29.22 & 14.72& 128.63\\
$(c)$ & \ding{51} & \ding{51} & \ding{51} & \ding{55} & 36.53& 31.15& 26.03& 22.04& 30.90& 29.33 & 14.74& 125.61\\
$(d)$ & \ding{51} & \ding{51} & \ding{55} & \ding{51} & 36.87 & 31.02 & 26.17 & 22.41 & 31.03 & 29.50 & 15.47& 133.56\\
$(e)$ & \ding{51} & \ding{55} & \ding{51}& \ding{51}& 36.84& 30.92& 26.06& 22.35& 31.00& 29.43& 15.53& 130.85\\
\rowcolor{orange}
$(f)$ & \ding{51} & \ding{51} & \ding{51} & \ding{51} & \textbf{37.48}& \textbf{31.21}& \textbf{26.49}& \textbf{22.92}& \textbf{31.06}& \textbf{29.83} & 15.50 & 130.55 \\
\bottomrule
\end{tabular}
\end{table*}

\begin{table*}[!ht]
\caption{Various prompt-feature interaction modules for CWP-Net. PSNR of different models is reported. FLOPs are measured on the patch size of $256 \times 256 \times 3$.}
\vspace{-0.2cm}
\label{tab:aba_interaction}
\centering
\begin{tabular}{l|ccccc|c|cc}%
\toprule 
Method & Derain & Dehaze & Deblur & Low-light & Denoise & Average & Params (M) & FLOPs (G)\\
\midrule
Concat$\&$Self-Attention & 37.22& 31.13& \textbf{26.51}& 22.24& 31.06& 29.63& 16.25 & 135.52 \\
Cross Attention & 36.93& \textbf{31.48}&26.25 & 22.51& 30.97& 29.63 & 15.64&  131.53\\
\rowcolor{orange}
SFT (ours) &\textbf{37.48}& 31.21& 26.49& \textbf{22.92}& \textbf{31.06}& \textbf{29.83}  & \textbf{15.50} & \textbf{130.55 }\\
\bottomrule
\end{tabular}
\vspace{-0.3cm}
\end{table*}

\begin{table}[!ht]
\caption{Ablation studies of numbers of clusters.}
\vspace{-0.2cm}
\label{tab:abla_clusters}
\centering
\begin{tabular}{cc>{\columncolor{orange}}cc}%
\toprule 
Number of clusters & 3 & 5 & 7\\
\midrule
Average PSNR & 29.77 & \textbf{29.83} & 29.80 \\
\bottomrule
\end{tabular}
\vspace{-0.3cm}
\end{table}

\subsection{Additional Results of Generalization Ability}
\label{sec:exp_unbiased}
Based on the discussion in Fig. \ref{fig_sc}, it is evident that the scenes of different tasks exhibit domain biases. Therefore, we refer to the original training data and test data as the imbalanced training set and imbalanced test set, respectively.
In order to demonstrate the generalization ability of our method in scenarios where spurious correlations between degradation patterns and semantic features do not exist, we construct a balanced test set for further evaluation. To remove domain bias in data, the clean images of the balanced test set are collected from the combination of the ground truth of each imbalanced test set. To ensure the consistency of degradation as much as possible, the degradation simulation method for rain and haze is consistent with the approach in \cite{yang2017deep} and \cite{li2018benchmarking}, respectively. We utilize gamma correction to simulate low-light environments with the detailed parameters that can be found in the all-in-one method \cite{kong2024towards}. To simulate motion blur, we utilize the popular blur kernels proposed in \cite{lee2020progressive} to convolve clear images. As for noisy images, Gaussian noise with three noise degrees is added consistent with Sect. \ref{sec:expri_set}.

Since models of the seven-pattern setting do not release their checkpoints, we verify the proposed method for AiOIR under the five-pattern setting and present the results in Table \ref{tab:contrast_unbiased}. Two all-in-one methods with official model weights provided are compared. Our CWP-Net has demonstrated superior performance over these two methods in all tasks regarding PSNR and SSIM, showcasing its ability to extract real degradation information that is irrelevant to semantics. In contrast, AirNet and PromptIR exhibit poor generalization performance in real scenarios where domain bias is not present.

This observation is further supported by our case studies in Fig. \ref{fig:case_study}. In the standard AiOIR setting, training datasets often contain scene-level sampling biases, where certain semantics frequently co-occur with specific degradations. When such biased co-occurrences no longer exist, as in our balanced test set, existing models tend to misinterpret semantic structures as degradation cues and consequently suffer severe performance drops. Visual examples in Fig. \ref{fig:case_study} demonstrate that these methods struggle to restore challenging scenes, such as indoor environments or crowded pedestrian areas, which are relatively rare in the original dehazing dataset. In contrast, CWP-Net achieves more reliable color and structure restoration in these cases. This improved generalization stems from our architecture design, where wavelet-based deconfounding suppresses spurious semantic–degradation dependencies by separating degradation-related information from semantic content. 

\begin{figure}
    \centering
    \includegraphics[width=\linewidth]{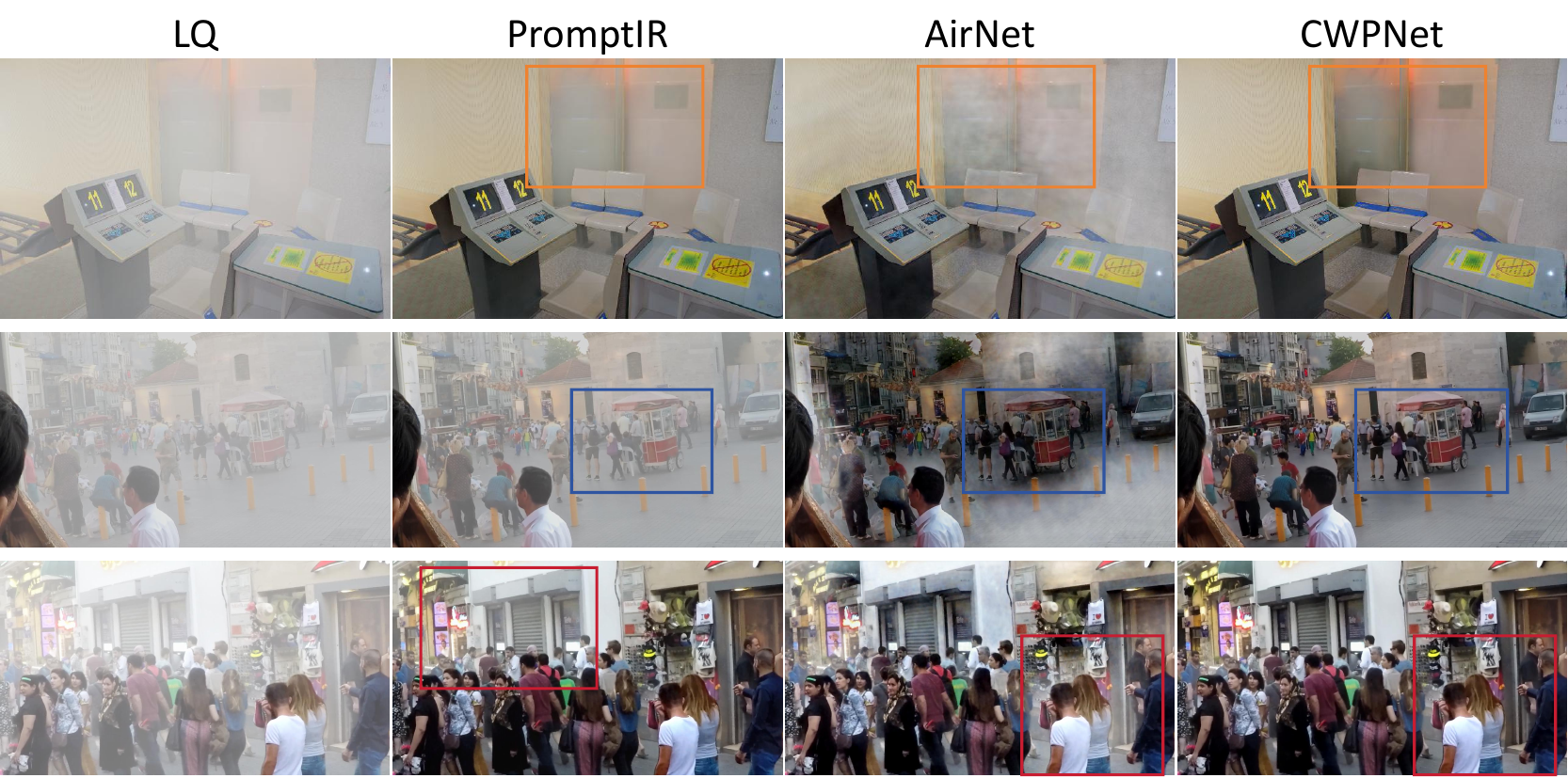}
    \caption{The case study on the generalization performance of CWP-Net under distribution shift.}
    \label{fig:case_study}
\end{figure}

\subsection{Ablation Study}
In this section, we conduct ablation studies to investigate the effect of each proposed component on the overall performance of our method. All studies are performed on a sampled version of the training data under the seven-pattern setting. The sampled data consists of 24,100 images, with the data ratio of each pattern being consistent with the original data.

\textbf{Effect of individual architecture modules.} We incorporate each module into the baseline gradually to show the impact of our proposed wavelet attention module and wavelet prompt block. The quantitative results are reported in Table \ref{tab:aba_individual}. The comparison between ``Model $(a)$'' and ``Model $(c)$'' demonstrates that the incorporation of the wavelet attention module of encoder and decoder can provide a performance gain over the baseline model in various tasks and achieve a 0.61dB overall PSNR improvement. This can be attributed to the explicit modeling of attention maps for different frequency subbands, which promotes the model to learn important specific frequency subbands in specific tasks.

Furthermore, we have also performed a complementary experiment by removing each module individually from the final CWP-Net to analyze the impact of each component in isolation. Removing the Wavelet Prompt Block (i.e., Model $(c)$) leads to the most significant performance drop (average –0.50 dB), highlighting its essential role in performing causal deconfounding through backdoor adjustment. Removing the WAE also results in a notable decrease, with an average 0.40dB PSNR drop. This is because WAE provides the degradation representation for the DWE module.

In variant $(e)$, we used raw features after CNN blocks (instead of the spatial attention map of the low-frequency parts) as input to the K-Means clustering process. These raw features contain both high-level semantic information and degradation features, making them poor degradation representations. The coupling of degradation representations across tasks explains the significant performance decline when WAE is removed. Moreover, removing WAD (Model (d)) causes a 0.33 dB average PSNR drop, confirming its importance in refining features during decoding.

\begin{table}
\centering
\caption{Efficiency comparisons of existing methods under the five-pattern setting. The inference time per image is calculated based on the average result of 100 forward propagations.}
\label{tab:eff1}

\begin{tabular}{l|ccc}
\toprule
Method & Params(M) & FLOPs(G) & Inference time(ms)\\
\midrule
BRDNet & 1.11 & 278 & 62\\
LPNet & 10.07 & 247.82 & 41.61\\
FDGAN & 11.80 & 69.16 & 17.94\\
MPRNet & 20.13 & 1707.36 & 84.21\\
\midrule
DL & 2.09 & 0.13 & 6.57\\
AirNet & 5.77 & 301.27 & 109.83\\
PromptIR & 32.97 & 158.14 & 82.93\\
Lin et al. & 112.05 & 332.08 & 128.51\\
\midrule
CWP-Net & 15.50 & 130.55 & 34.32\\
\bottomrule
\end{tabular}
\end{table}

\begin{table}
\centering
\caption{Efficiency comparisons of existing methods under the seven-pattern setting. The inference time per image is calculated based on the average result of 100 forward propagations.}
\label{tab:eff2}

\begin{tabular}{l|ccc}
\toprule
Method & Params(M) & FLOPs(G) & Inference time(ms)\\
\midrule
NAFNet & 17.11& 15.96& 18.90\\
HINet & 88.67& 170.49& 18.20\\
MPRNet & 15.74& 850.04& 60.52\\
DGUNet & 17.33& 866.17& 105.10\\
MIRNetv2 & 5.86& 140.92& 46.41\\
SwinIR & 0.91& 58.72& 122.46\\
Restormer & 26.13& 140.99& 76.04\\
\midrule
DL & 2.09& 0.13 & 6.57\\
Transweather & 37.93& 6.13& 10.18\\
TAPE & 1.07& 32.26&  37.80\\
AirNet & 8.93& 350.69& 153.24\\
IDR & 15.34& 114.18 & 71.29\\
\midrule
CWP-Net & 15.50& 130.55& 34.32\\
\bottomrule
\end{tabular}
\end{table}

\textbf{Prompt-feature interaction.} To investigate the effect of various prompt-feature interaction modules on our proposed CWP-Net, we substitute the SFT module with self-attention and cross-attention for comparison. Specifically, learnable prompt components $L_{j}$ and wavelet features $Z_{j}$ are concatenated along the channel dimension before being fed into the self-attention module for modulation. With regard to the cross-attention module, $L_{j}$ is utilized as the \textit{Query}, and $Z_{j}$ is utilized as the \textit{Key} and \textit{Value}. Comparison results are shown in Table \ref{tab:aba_interaction}. We also report the parameters and FLOPs for different methods. Since the self-attention and cross-attention modules model global dependencies, their performance is suboptimal in scenarios where local textures need to be reconstructed (such as deraining and denoising). Considering the average performance on all tasks along with the computational complexity, we select the SFT module as the implementation of our prompt-feature interaction procedure.

\textbf{Number of clusters.} Fig. \ref{fig:alternative_variable} presents the statistics of degradation-mild subbands for different degradation pattern datasets. However, it is worth noting that even different samples of the same task will show different degradation-mild subbands. For example, rain streaks with an angle less than 45 degrees degrade the horizontal subband more, while rain streaks with an angle greater than 45 degrees degrade the vertical subband. Thus, the number of clusters does not necessarily have to align strictly with the number of degradation patterns defined in this paper. To this end, we have studied the influence of cluster numbers on overall performance. As indicated in Table \ref{tab:abla_clusters},  increasing the number of clusters from five to seven results in a slight decrease in performance. Ultimately, we choose to set the number of clusters to five.

\subsection{Efficiency Analysis}

To evaluate the practical applicability of CWP-Net, we conduct a comprehensive analysis of model complexity, including the number of parameters, FLOPs, and runtime. All measurements are performed on a single NVIDIA RTX 3090 GPU using $256\times256\times3$ resolution images to ensure fair comparison across methods.

Table \ref{tab:eff1} and Table \ref{tab:eff2} present the computational complexity comparison with both general image restoration methods and state-of-the-art AiOIR approaches. Compared to Lin et al. and PromptIR, CWP-Net delivers comparable performance with fewer parameters and faster inference speed. While some lightweight methods (e.g., BRDNet, DL, and TAPE) achieve lower computational complexity, they are not specifically designed for AiOIR, resulting in significant performance degradation. CWP-Net achieves 5.35dB higher average PSNR than BRDNet, demonstrating that our method delivers substantially better restoration quality despite moderate increases in computational requirements.

When compared to methods with similar parameter counts, such as IDR for AiOIR or general restoration methods like MPRNet and DGUNet, CWP-Net consistently demonstrates superior restoration performance. Combined with the parameters analysis in Table \ref{tab:aba_individual}, it can be concluded that the additional components in CWP-Net introduce reasonable computational overhead while delivering performance improvements. 

\subsection{Limitations}
While wavelet decomposition offers powerful frequency-domain analysis, it inherently lacks contextual scene understanding. Specifically, in scenarios where degradation patterns visually resemble semantic content (e.g., rain streaks on a chessboard), the model may struggle to disentangle degradations from meaningful edges fully. During dehazing, the absence of an explicit spatial-depth coupling mechanism limits the model's ability to recover true color tones and depth-consistent brightness.

These limitations highlight opportunities for future work, such as integrating frequency-based methods with large-scale vision foundation models that offer strong spatial and semantic reasoning capabilities. Additionally, extending the framework to handle mixed or compound degradations within a single image remains an open challenge.

\section{Conclusion}
We construct the SCM to theoretically model the inherent causal mechanism of AiOIR methods and introduce the experiments to demonstrate the intrinsic reasons exacerbating the performance of blind AiOIR methods: the spurious correlation between non-degradation semantic features and degradation patterns, and the biased estimation of degradation patterns. To this end, we provide the causality-guided methodology to address the inherent defects of existing methods. According to the practical implementation, we propose CWP-Net, incorporating WAE, WAD, and WPB, to explore the deconfounded causation between the input degraded image and restored image, thereby improving the effectiveness and generalization of AiOIR models. Extensive experiments demonstrate that the proposed CWP-Net can consistently improve the performance of AiOIR models on two settings.

\bibliographystyle{IEEEtran}
\bibliography{egbib}

@article{tian2020image,
  title={Image denoising using deep CNN with batch renormalization},
  author={Tian, Chunwei and Xu, Yong and Zuo, Wangmeng},
  journal={Neural Networks},
  volume={121},
  pages={461--473},
  year={2020},
  publisher={Elsevier}
}

@inproceedings{gao2019dynamic,
  title={Dynamic scene deblurring with parameter selective sharing and nested skip connections},
  author={Gao, Hongyun and Tao, Xin and Shen, Xiaoyong and Jia, Jiaya},
  booktitle={Proceedings of the IEEE/CVF conference on computer vision and pattern recognition},
  pages={3848--3856},
  year={2019}
}

@inproceedings{dong2020fd,
  title={FD-GAN: Generative adversarial networks with fusion-discriminator for single image dehazing},
  author={Dong, Yu and Liu, Yihao and Zhang, He and Chen, Shifeng and Qiao, Yu},
  booktitle={Proceedings of the AAAI conference on artificial intelligence},
  volume={34},
  number={07},
  pages={10729--10736},
  year={2020}
}

@inproceedings{zamir2021multi,
  title={Multi-stage progressive image restoration},
  author={Zamir, Syed Waqas and Arora, Aditya and Khan, Salman and Hayat, Munawar and Khan, Fahad Shahbaz and Yang, Ming-Hsuan and Shao, Ling},
  booktitle={Proceedings of the IEEE/CVF conference on computer vision and pattern recognition},
  pages={14821--14831},
  year={2021}
}

@article{fan2019general,
  title={A general decoupled learning framework for parameterized image operators},
  author={Fan, Qingnan and Chen, Dongdong and Yuan, Lu and Hua, Gang and Yu, Nenghai and Chen, Baoquan},
  journal={IEEE transactions on pattern analysis and machine intelligence},
  volume={43},
  number={1},
  pages={33--47},
  year={2019},
  publisher={IEEE}
}

@inproceedings{li2022all,
  title={All-in-one image restoration for unknown corruption},
  author={Li, Boyun and Liu, Xiao and Hu, Peng and Wu, Zhongqin and Lv, Jiancheng and Peng, Xi},
  booktitle={Proceedings of the IEEE/CVF Conference on Computer Vision and Pattern Recognition},
  pages={17452--17462},
  year={2022}
}

@article{potlapalli2024promptir,
  title={PromptIR: Prompting for All-in-One Image Restoration},
  author={Potlapalli, Vaishnav and Zamir, Syed Waqas and Khan, Salman H and Shahbaz Khan, Fahad},
  journal={Advances in Neural Information Processing Systems},
  volume={36},
  year={2024}
}

@article{lin2023improving,
  title={Improving image restoration through removing degradations in textual representations},
  author={Lin, Jingbo and Zhang, Zhilu and Wei, Yuxiang and Ren, Dongwei and Jiang, Dongsheng and Zuo, Wangmeng},
  journal={arXiv preprint arXiv:2312.17334},
  year={2023}
}

@inproceedings{chen2023learning,
  title={Learning a sparse transformer network for effective image deraining},
  author={Chen, Xiang and Li, Hao and Li, Mingqiang and Pan, Jinshan},
  booktitle={Proceedings of the IEEE/CVF Conference on Computer Vision and Pattern Recognition},
  pages={5896--5905},
  year={2023}
}

@article{song2023vision,
  title={Vision transformers for single image dehazing},
  author={Song, Yuda and He, Zhuqing and Qian, Hui and Du, Xin},
  journal={IEEE Transactions on Image Processing},
  volume={32},
  pages={1927--1941},
  year={2023},
  publisher={IEEE}
}

@inproceedings{li2023efficient,
  title={Efficient and explicit modelling of image hierarchies for image restoration},
  author={Li, Yawei and Fan, Yuchen and Xiang, Xiaoyu and Demandolx, Denis and Ranjan, Rakesh and Timofte, Radu and Van Gool, Luc},
  booktitle={Proceedings of the IEEE/CVF Conference on Computer Vision and Pattern Recognition},
  pages={18278--18289},
  year={2023}
}

@inproceedings{cui2023focal,
  title={Focal network for image restoration},
  author={Cui, Yuning and Ren, Wenqi and Cao, Xiaochun and Knoll, Alois},
  booktitle={Proceedings of the IEEE/CVF international conference on computer vision},
  pages={13001--13011},
  year={2023}
}

@inproceedings{zamir2022restormer,
  title={Restormer: Efficient transformer for high-resolution image restoration},
  author={Zamir, Syed Waqas and Arora, Aditya and Khan, Salman and Hayat, Munawar and Khan, Fahad Shahbaz and Yang, Ming-Hsuan},
  booktitle={Proceedings of the IEEE/CVF conference on computer vision and pattern recognition},
  pages={5728--5739},
  year={2022}
}

@article{wang2024survey,
  title={A survey on facial image deblurring},
  author={Wang, Bingnan and Xu, Fanjiang and Zheng, Quan},
  journal={Computational Visual Media},
  volume={10},
  number={1},
  pages={3--25},
  year={2024},
  publisher={Springer}
}

@inproceedings{li2020all,
  title={All in one bad weather removal using architectural search},
  author={Li, Ruoteng and Tan, Robby T and Cheong, Loong-Fah},
  booktitle={Proceedings of the IEEE/CVF conference on computer vision and pattern recognition},
  pages={3175--3185},
  year={2020}
}

@article{luo2023controlling,
  title={Controlling vision-language models for universal image restoration},
  author={Luo, Ziwei and Gustafsson, Fredrik K and Zhao, Zheng and Sj{\"o}lund, Jens and Sch{\"o}n, Thomas B},
  journal={arXiv preprint arXiv:2310.01018},
  year={2023}
}

@inproceedings{chen2021pre,
  title={Pre-trained image processing transformer},
  author={Chen, Hanting and Wang, Yunhe and Guo, Tianyu and Xu, Chang and Deng, Yiping and Liu, Zhenhua and Ma, Siwei and Xu, Chunjing and Xu, Chao and Gao, Wen},
  booktitle={Proceedings of the IEEE/CVF conference on computer vision and pattern recognition},
  pages={12299--12310},
  year={2021}
}

@inproceedings{zhu2023learning,
  title={Learning weather-general and weather-specific features for image restoration under multiple adverse weather conditions},
  author={Zhu, Yurui and Wang, Tianyu and Fu, Xueyang and Yang, Xuanyu and Guo, Xin and Dai, Jifeng and Qiao, Yu and Hu, Xiaowei},
  booktitle={Proceedings of the IEEE/CVF Conference on Computer Vision and Pattern Recognition},
  pages={21747--21758},
  year={2023}
}

@inproceedings{dong2023multi,
  title={Multi-scale Residual Low-Pass Filter Network for Image Deblurring},
  author={Dong, Jiangxin and Pan, Jinshan and Yang, Zhongbao and Tang, Jinhui},
  booktitle={Proceedings of the IEEE/CVF International Conference on Computer Vision},
  pages={12345--12354},
  year={2023}
}

@article{zhang2024pan,
  title={Pan-Sharpening With Wavelet-Enhanced High-Frequency Information},
  author={Zhang, Jie and He, Xuanhua and Yan, Keyu and Cao, Ke and Li, Rui and Xie, Chengjun and Zhou, Man and Hong, Danfeng},
  journal={IEEE Transactions on Geoscience and Remote Sensing},
  volume={62},
  pages={1--14},
  year={2024},
  publisher={IEEE}
}

@inproceedings{zou2021sdwnet,
  title={Sdwnet: A straight dilated network with wavelet transformation for image deblurring},
  author={Zou, Wenbin and Jiang, Mingchao and Zhang, Yunchen and Chen, Liang and Lu, Zhiyong and Wu, Yi},
  booktitle={Proceedings of the IEEE/CVF international conference on computer vision},
  pages={1895--1904},
  year={2021}
}

@article{zhang2023underwater,
  title={Underwater image enhancement via weighted wavelet visual perception fusion},
  author={Zhang, Weidong and Zhou, Ling and Zhuang, Peixian and Li, Guohou and Pan, Xipeng and Zhao, Wenyi and Li, Chongyi},
  journal={IEEE Transactions on Circuits and Systems for Video Technology},
  year={2023},
  publisher={IEEE}
}

@inproceedings{zhang2023ingredient,
  title={Ingredient-oriented multi-degradation learning for image restoration},
  author={Zhang, Jinghao and Huang, Jie and Yao, Mingde and Yang, Zizheng and Yu, Hu and Zhou, Man and Zhao, Feng},
  booktitle={Proceedings of the IEEE/CVF Conference on Computer Vision and Pattern Recognition},
  pages={5825--5835},
  year={2023}
}

@article{ma2016waterloo,
  title={Waterloo exploration database: New challenges for image quality assessment models},
  author={Ma, Kede and Duanmu, Zhengfang and Wu, Qingbo and Wang, Zhou and Yong, Hongwei and Li, Hongliang and Zhang, Lei},
  journal={IEEE Transactions on Image Processing},
  volume={26},
  number={2},
  pages={1004--1016},
  year={2016},
  publisher={IEEE}
}

@article{arbelaez2010contour,
  title={Contour detection and hierarchical image segmentation},
  author={Arbelaez, Pablo and Maire, Michael and Fowlkes, Charless and Malik, Jitendra},
  journal={IEEE transactions on pattern analysis and machine intelligence},
  volume={33},
  number={5},
  pages={898--916},
  year={2010},
  publisher={IEEE}
}

@inproceedings{martin2001database,
  title={A database of human segmented natural images and its application to evaluating segmentation algorithms and measuring ecological statistics},
  author={Martin, David and Fowlkes, Charless and Tal, Doron and Malik, Jitendra},
  booktitle={Proceedings Eighth IEEE International Conference on Computer Vision. ICCV 2001},
  volume={2},
  pages={416--423},
  year={2001},
  organization={IEEE}
}

@inproceedings{yang2020learning,
  title={Learning texture transformer network for image super-resolution},
  author={Yang, Fuzhi and Yang, Huan and Fu, Jianlong and Lu, Hongtao and Guo, Baining},
  booktitle={Proceedings of the IEEE/CVF conference on computer vision and pattern recognition},
  pages={5791--5800},
  year={2020}
}

@article{li2018benchmarking,
  title={Benchmarking single-image dehazing and beyond},
  author={Li, Boyi and Ren, Wenqi and Fu, Dengpan and Tao, Dacheng and Feng, Dan and Zeng, Wenjun and Wang, Zhangyang},
  journal={IEEE Transactions on Image Processing},
  volume={28},
  number={1},
  pages={492--505},
  year={2018},
  publisher={IEEE}
}

@article{wei2018deep,
  title={Deep retinex decomposition for low-light enhancement},
  author={Wei, Chen and Wang, Wenjing and Yang, Wenhan and Liu, Jiaying},
  journal={arXiv preprint arXiv:1808.04560},
  year={2018}
}

@book{pearl2009causality,
  title={Causality},
  author={Pearl, Judea},
  year={2009},
  publisher={Cambridge university press}
}

@book{pearl2016causal,
  title={Causal inference in statistics: A primer},
  author={Pearl, Judea and Glymour, Madelyn and Jewell, Nicholas P},
  year={2016},
  publisher={John Wiley \& Sons}
}

@article{li2021online,
  title={Online rain/snow removal from surveillance videos},
  author={Li, Minghan and Cao, Xiangyong and Zhao, Qian and Zhang, Lei and Meng, Deyu},
  journal={IEEE Transactions on Image Processing},
  volume={30},
  pages={2029--2044},
  year={2021},
  publisher={IEEE}
}

@article{li2021deep,
  title={A deep learning based image enhancement approach for autonomous driving at night},
  author={Li, Guofa and Yang, Yifan and Qu, Xingda and Cao, Dongpu and Li, Keqiang},
  journal={Knowledge-Based Systems},
  volume={213},
  pages={106617},
  year={2021},
  publisher={Elsevier}
}

@article{xu2022illumination,
  title={Illumination guided attentive wavelet network for low-light image enhancement},
  author={Xu, Jingzhao and Yuan, Mengke and Yan, Dong-Ming and Wu, Tieru},
  journal={IEEE Transactions on Multimedia},
  volume={25},
  pages={6258--6271},
  year={2022},
  publisher={IEEE}
}

@inproceedings{yu2021wavefill,
  title={Wavefill: A wavelet-based generation network for image inpainting},
  author={Yu, Yingchen and Zhan, Fangneng and Lu, Shijian and Pan, Jianxiong and Ma, Feiying and Xie, Xuansong and Miao, Chunyan},
  booktitle={Proceedings of the IEEE/CVF international conference on computer vision},
  pages={14114--14123},
  year={2021}
}

@article{lai2022face,
  title={Face deblurring using dual camera fusion on mobile phones},
  author={Lai, Wei-Sheng and Shih, Yichang and Chu, Lun-Cheng and Wu, Xiaotong and Tsai, Sung-Fang and Krainin, Michael and Sun, Deqing and Liang, Chia-Kai},
  journal={ACM Transactions on Graphics (TOG)},
  volume={41},
  number={4},
  pages={1--16},
  year={2022},
  publisher={ACM New York, NY, USA}
}

@inproceedings{li2019heavy,
  title={Heavy rain image restoration: Integrating physics model and conditional adversarial learning},
  author={Li, Ruoteng and Cheong, Loong-Fah and Tan, Robby T},
  booktitle={Proceedings of the IEEE/CVF conference on computer vision and pattern recognition},
  pages={1633--1642},
  year={2019}
}

@inproceedings{zheng2023curricular,
  title={Curricular contrastive regularization for physics-aware single image dehazing},
  author={Zheng, Yu and Zhan, Jiahui and He, Shengfeng and Dong, Junyu and Du, Yong},
  booktitle={Proceedings of the IEEE/CVF conference on computer vision and pattern recognition},
  pages={5785--5794},
  year={2023}
}

@article{zhang2021exposure,
  title={Exposure trajectory recovery from motion blur},
  author={Zhang, Youjian and Wang, Chaoyue and Maybank, Stephen J and Tao, Dacheng},
  journal={IEEE Transactions on Pattern Analysis and Machine Intelligence},
  volume={44},
  number={11},
  pages={7490--7504},
  year={2021},
  publisher={IEEE}
}

@inproceedings{fei2023generative,
  title={Generative diffusion prior for unified image restoration and enhancement},
  author={Fei, Ben and Lyu, Zhaoyang and Pan, Liang and Zhang, Junzhe and Yang, Weidong and Luo, Tianyue and Zhang, Bo and Dai, Bo},
  booktitle={Proceedings of the IEEE/CVF Conference on Computer Vision and Pattern Recognition},
  pages={9935--9946},
  year={2023}
}

@article{guo2024mambair,
  title={Mambair: A simple baseline for image restoration with state-space model},
  author={Guo, Hang and Li, Jinmin and Dai, Tao and Ouyang, Zhihao and Ren, Xudong and Xia, Shu-Tao},
  journal={arXiv preprint arXiv:2402.15648},
  year={2024}
}

@inproceedings{park2023all,
  title={All-in-one image restoration for unknown degradations using adaptive discriminative filters for specific degradations},
  author={Park, Dongwon and Lee, Byung Hyun and Chun, Se Young},
  booktitle={2023 IEEE/CVF Conference on Computer Vision and Pattern Recognition (CVPR)},
  pages={5815--5824},
  year={2023},
  organization={IEEE}
}

@article{li2023prompt,
  title={Prompt-in-prompt learning for universal image restoration},
  author={Li, Zilong and Lei, Yiming and Ma, Chenglong and Zhang, Junping and Shan, Hongming},
  journal={arXiv preprint arXiv:2312.05038},
  year={2023}
}

@inproceedings{he2016deep,
  title={Deep residual learning for image recognition},
  author={He, Kaiming and Zhang, Xiangyu and Ren, Shaoqing and Sun, Jian},
  booktitle={Proceedings of the IEEE conference on computer vision and pattern recognition},
  pages={770--778},
  year={2016}
}

@article{simonyan2014very,
  title={Very deep convolutional networks for large-scale image recognition},
  author={Simonyan, Karen and Zisserman, Andrew},
  journal={arXiv preprint arXiv:1409.1556},
  year={2014}
}

@inproceedings{cui2023irnext,
  title={Irnext: Rethinking convolutional network design for image restoration},
  author={Cui, Yuning and Ren, Wenqi and Yang, Sining and Cao, Xiaochun and Knoll, Alois},
  booktitle={International conference on machine learning},
  year={2023}
}

@inproceedings{cho2021rethinking,
  title={Rethinking coarse-to-fine approach in single image deblurring},
  author={Cho, Sung-Jin and Ji, Seo-Won and Hong, Jun-Pyo and Jung, Seung-Won and Ko, Sung-Jea},
  booktitle={Proceedings of the IEEE/CVF international conference on computer vision},
  pages={4641--4650},
  year={2021}
}

@inproceedings{cui2023selective,
  title={Selective frequency network for image restoration},
  author={Cui, Yuning and Tao, Yi and Bing, Zhenshan and Ren, Wenqi and Gao, Xinwei and Cao, Xiaochun and Huang, Kai and Knoll, Alois},
  booktitle={The Eleventh International Conference on Learning Representations},
  year={2023}
}

@inproceedings{tu2022maxim,
  title={Maxim: Multi-axis mlp for image processing},
  author={Tu, Zhengzhong and Talebi, Hossein and Zhang, Han and Yang, Feng and Milanfar, Peyman and Bovik, Alan and Li, Yinxiao},
  booktitle={Proceedings of the IEEE/CVF conference on computer vision and pattern recognition},
  pages={5769--5780},
  year={2022}
}

@inproceedings{huang2017wavelet,
  title={Wavelet-srnet: A wavelet-based cnn for multi-scale face super resolution},
  author={Huang, Huaibo and He, Ran and Sun, Zhenan and Tan, Tieniu},
  booktitle={Proceedings of the IEEE international conference on computer vision},
  pages={1689--1697},
  year={2017}
}

@inproceedings{woo2018cbam,
  title={Cbam: Convolutional block attention module},
  author={Woo, Sanghyun and Park, Jongchan and Lee, Joon-Young and Kweon, In So},
  booktitle={Proceedings of the European conference on computer vision (ECCV)},
  pages={3--19},
  year={2018}
}

@article{verma1993graphical,
  title={Graphical aspects of causal models},
  author={Verma, Thomas},
  journal={Technical R eport R-191, UCLA},
  year={1993}
}

@incollection{bareinboim2022pearl,
  title={On Pearl’s hierarchy and the foundations of causal inference},
  author={Bareinboim, Elias and Correa, Juan D and Ibeling, Duligur and Icard, Thomas},
  booktitle={Probabilistic and causal inference: the works of judea pearl},
  pages={507--556},
  year={2022}
}

@article{pearl2011graphical,
  title={Graphical models, causality, and intervention},
  author={Pearl, Judea},
  journal={},
  year={2011}
}

@inproceedings{wang2018recovering,
  title={Recovering realistic texture in image super-resolution by deep spatial feature transform},
  author={Wang, Xintao and Yu, Ke and Dong, Chao and Loy, Chen Change},
  booktitle={Proceedings of the IEEE conference on computer vision and pattern recognition},
  pages={606--615},
  year={2018}
}

@article{wang2004image,
  title={Image quality assessment: from error visibility to structural similarity},
  author={Wang, Zhou and Bovik, Alan C and Sheikh, Hamid R and Simoncelli, Eero P},
  journal={IEEE transactions on image processing},
  volume={13},
  number={4},
  pages={600--612},
  year={2004},
  publisher={IEEE}
}

@inproceedings{chen2022simple,
  title={Simple baselines for image restoration},
  author={Chen, Liangyu and Chu, Xiaojie and Zhang, Xiangyu and Sun, Jian},
  booktitle={European conference on computer vision},
  pages={17--33},
  year={2022},
  organization={Springer}
}

@inproceedings{chen2021hinet,
  title={Hinet: Half instance normalization network for image restoration},
  author={Chen, Liangyu and Lu, Xin and Zhang, Jie and Chu, Xiaojie and Chen, Chengpeng},
  booktitle={Proceedings of the IEEE/CVF conference on computer vision and pattern recognition},
  pages={182--192},
  year={2021}
}

@article{zamir2022learning,
  title={Learning enriched features for fast image restoration and enhancement},
  author={Zamir, Syed Waqas and Arora, Aditya and Khan, Salman and Hayat, Munawar and Khan, Fahad Shahbaz and Yang, Ming-Hsuan and Shao, Ling},
  journal={IEEE transactions on pattern analysis and machine intelligence},
  volume={45},
  number={2},
  pages={1934--1948},
  year={2022},
  publisher={IEEE}
}

@inproceedings{mou2022deep,
  title={Deep generalized unfolding networks for image restoration},
  author={Mou, Chong and Wang, Qian and Zhang, Jian},
  booktitle={Proceedings of the IEEE/CVF conference on computer vision and pattern recognition},
  pages={17399--17410},
  year={2022}
}

@inproceedings{liang2021swinir,
  title={Swinir: Image restoration using swin transformer},
  author={Liang, Jingyun and Cao, Jiezhang and Sun, Guolei and Zhang, Kai and Van Gool, Luc and Timofte, Radu},
  booktitle={Proceedings of the IEEE/CVF international conference on computer vision},
  pages={1833--1844},
  year={2021}
}

@inproceedings{liu2022tape,
  title={Tape: Task-agnostic prior embedding for image restoration},
  author={Liu, Lin and Xie, Lingxi and Zhang, Xiaopeng and Yuan, Shanxin and Chen, Xiangyu and Zhou, Wengang and Li, Houqiang and Tian, Qi},
  booktitle={European Conference on Computer Vision},
  pages={447--464},
  year={2022},
  organization={Springer}
}

@inproceedings{valanarasu2022transweather,
  title={Transweather: Transformer-based restoration of images degraded by adverse weather conditions},
  author={Valanarasu, Jeya Maria Jose and Yasarla, Rajeev and Patel, Vishal M},
  booktitle={Proceedings of the IEEE/CVF Conference on Computer Vision and Pattern Recognition},
  pages={2353--2363},
  year={2022}
}

@article{huynh2008scope,
  title={Scope of validity of PSNR in image/video quality assessment},
  author={Huynh-Thu, Quan and Ghanbari, Mohammed},
  journal={Electronics letters},
  volume={44},
  number={13},
  pages={800--801},
  year={2008},
  publisher={IET}
}

@inproceedings{yang2017deep,
  title={Deep joint rain detection and removal from a single image},
  author={Yang, Wenhan and Tan, Robby T and Feng, Jiashi and Liu, Jiaying and Guo, Zongming and Yan, Shuicheng},
  booktitle={Proceedings of the IEEE conference on computer vision and pattern recognition},
  pages={1357--1366},
  year={2017}
}

@article{lee2020progressive,
  title={Progressive semantic face deblurring},
  author={Lee, Tae Bok and Jung, Soo Hyun and Heo, Yong Seok},
  journal={IEEE Access},
  volume={8},
  pages={223548--223561},
  year={2020},
  publisher={IEEE}
}

@article{kong2024towards,
  title={Towards Effective Multiple-in-One Image Restoration: A Sequential and Prompt Learning Strategy},
  author={Kong, Xiangtao and Dong, Chao and Zhang, Lei},
  journal={arXiv preprint arXiv:2401.03379},
  year={2024}
}

@book{jain1988algorithms,
  title={Algorithms for clustering data},
  author={Jain, Anil K and Dubes, Richard C},
  year={1988},
  publisher={Prentice-Hall, Inc.}
}

@inproceedings{rombach2022high,
  title={High-resolution image synthesis with latent diffusion models},
  author={Rombach, Robin and Blattmann, Andreas and Lorenz, Dominik and Esser, Patrick and Ommer, Bj{\"o}rn},
  booktitle={Proceedings of the IEEE/CVF conference on computer vision and pattern recognition},
  pages={10684--10695},
  year={2022}
}

@inproceedings{radford2021learning,
  title={Learning transferable visual models from natural language supervision},
  author={Radford, Alec and Kim, Jong Wook and Hallacy, Chris and Ramesh, Aditya and Goh, Gabriel and Agarwal, Sandhini and Sastry, Girish and Askell, Amanda and Mishkin, Pamela and Clark, Jack and others},
  booktitle={International conference on machine learning},
  pages={8748--8763},
  year={2021},
  organization={PMLR}
}

@InProceedings{Nah_2017_CVPR,
author = {Nah, Seungjun and Hyun Kim, Tae and Mu Lee, Kyoung},
title = {Deep Multi-Scale Convolutional Neural Network for Dynamic Scene Deblurring},
booktitle = {Proceedings of the IEEE Conference on Computer Vision and Pattern Recognition (CVPR)},
month = {July},
year = {2017}
}

@article{xiao2025incorporating,
  title={Incorporating degradation estimation in light field spatial super-resolution},
  author={Xiao, Zeyu and Xiong, Zhiwei},
  journal={Computer Vision and Image Understanding},
  volume={252},
  pages={104295},
  year={2025},
  publisher={Elsevier}
}

@article{xiao2023dive,
  title={A dive into sam prior in image restoration},
  author={Xiao, Zeyu and Bai, Jiawang and Lu, Zhihe and Xiong, Zhiwei},
  journal={arXiv preprint arXiv:2305.13620},
  year={2023}
}

@article{refr3.1,
  title={MC-Blur: A comprehensive benchmark for image deblurring},
  author={Zhang, Kaihao and Wang, Tao and Luo, Wenhan and Ren, Wenqi and Stenger, Bj{\"o}rn and Liu, Wei and Li, Hongdong and Yang, Ming-Hsuan},
  journal={IEEE Transactions on Circuits and Systems for Video Technology},
  volume={34},
  number={5},
  pages={3755--3767},
  year={2023},
  publisher={IEEE}
}

@article{refr3.2,
  title={Adversarial spatio-temporal learning for video deblurring},
  author={Zhang, Kaihao and Luo, Wenhan and Zhong, Yiran and Ma, Lin and Liu, Wei and Li, Hongdong},
  journal={IEEE Transactions on Image Processing},
  volume={28},
  number={1},
  pages={291--301},
  year={2018},
  publisher={IEEE}
}

@article{refr3.4,
  title={Enhanced spatio-temporal interaction learning for video deraining: faster and better},
  author={Zhang, Kaihao and Li, Dongxu and Luo, Wenhan and Ren, Wenqi and Liu, Wei},
  journal={IEEE Transactions on Pattern Analysis and Machine Intelligence},
  volume={45},
  number={1},
  pages={1287--1293},
  year={2022},
  publisher={IEEE}
}

@article{refr3.5,
  title={MB-TaylorFormer V2: improved multi-branch linear transformer expanded by Taylor formula for image restoration},
  author={Jin, Zhi and Qiu, Yuwei and Zhang, Kaihao and Li, Hongdong and Luo, Wenhan},
  journal={IEEE Transactions on Pattern Analysis and Machine Intelligence},
  year={2025},
  publisher={IEEE}
}

@article{refr3.6,
  title={LLDiffusion: Learning degradation representations in diffusion models for low-light image enhancement},
  author={Wang, Tao and Zhang, Kaihao and Zhang, Yong and Luo, Wenhan and Stenger, Bj{\"o}rn and Lu, Tong and Kim, Tae-Kyun and Liu, Wei},
  journal={Pattern Recognition},
  volume={166},
  pages={111628},
  year={2025},
  publisher={Elsevier}
}

@article{refr3.7,
  title={All-in-one Weather-degraded Image Restoration via Adaptive Degradation-aware Self-prompting Model},
  author={Wen, Yuanbo and Gao, Tao and Li, Ziqi and Zhang, Jing and Zhang, Kaihao and Chen, Ting},
  journal={IEEE Transactions on Multimedia},
  year={2025},
  publisher={IEEE}
}

@article{spirtes2013causal,
  title={Causal inference in the presence of latent variables and selection bias},
  author={Spirtes, Peter L and Meek, Christopher and Richardson, Thomas S},
  journal={arXiv preprint arXiv:1302.4983},
  year={2013}
}

@book{spirtes2000causation,
  title={Causation, prediction, and search},
  author={Spirtes, Peter and Glymour, Clark N and Scheines, Richard},
  year={2000},
  publisher={MIT press}
}

@article{peters2016causal,
  title={Causal inference by using invariant prediction: identification and confidence intervals},
  author={Peters, Jonas and B{\"u}hlmann, Peter and Meinshausen, Nicolai},
  journal={Journal of the Royal Statistical Society Series B: Statistical Methodology},
  volume={78},
  number={5},
  pages={947--1012},
  year={2016},
  publisher={Oxford University Press}
}

@article{zhang2012kernel,
  title={Kernel-based conditional independence test and application in causal discovery},
  author={Zhang, Kun and Peters, Jonas and Janzing, Dominik and Sch{\"o}lkopf, Bernhard},
  journal={arXiv preprint arXiv:1202.3775},
  year={2012}
}

\appendix
\label{appendix}
\subsection{The Background of Causality}\label{background}
We will provide a detailed preliminary of the causal background knowledge required for our work. For further information, please refer to the literature \cite{pearl2009causality,pearl2016causal,bareinboim2022pearl}.

\textbf{Structural Causal Models.} To formally model the causal relationships between variables, we provide the definition of a structural causal model (SCM) below. All causal modeling analyses in AiOIR are based on the SCM framework.
\begin{definition}
    \label{SCM} \textbf{(Structural Causal Model)}
    A structural causal model $\mathcal{M}$ is a triple $\mathcal{M}=\left \langle U,V,F \right \rangle$, where:
\begin{enumerate}
    \item $U$ is the set of background variables, called \textit{exogenous variables}, which are determined by external factors outside the model and whose variations do not need to be explained.
    \item $V$ is a set of variables $\left\{ V_1,V_2,...,V_n \right\}$, called \textit{endogenous variables}, which are the internal variables of the model we are interested in studying to understand their causal relationships.
    \item $F$ is a set of functions $\left\{ f_1,f_2,...,f_n \right\}$ where each $f_i$ is a mapping from $U_i \cup Pa_i$ to $V_i$, with $U_i \subseteq U$ and $Pa_i \subseteq V \backslash V_i$. For each $f_i \in F$, we have $v_i = f_i(pa_i, u_i)$. That is, the function $f_i$ assigns a value to the variable $V_i$ based on the values of other variables in the model.
\end{enumerate}
\end{definition}
Here, we adopt the definition of \textit{causality}: if $V_j$ is in the domain of $f_i$, then variable $V_j$ is a \textit{direct cause} of variable $V_i$. If $V_j$ is a direct cause of $V_i$ or an indirect cause through other variables, then $V_j$ is a \textit{cause} of $V_i$. Note that in this study, we ignore exogenous variables $U$ and focus solely on the intrinsic causal relationships within the model.

Every SCM $\mathcal{M}$ has a corresponding directed graph $\mathcal{G}$, where the nodes in $\mathcal{G}$ represent the variables in $U \cup V$, and the directed edges between nodes represent the functions in $F$. In the graph $\mathcal{G}$, if node $V_i$ is a descendant of node $V_j$, then $V_j$ is a cause of $V_i$.

\textbf{Interventions and Do-calculus.} The causal model $\mathcal{M}$ describes intrinsic causal mechanisms, characterized by the observed distribution $\mathbb{P}_{\mathcal{M}}(V)=\prod_{i=1}^{n}\mathbb{P}(v_i\mid pa_i)$. Intervention\footnote{The definition here refers to the atomic intervention \cite{pearl2009causality}. For brevity, we intervene on only one variable.} is defined as forcing a variable $V_i$ to take on a fixed value $v$, modifying the model $\mathcal{M}=\left \langle U,V,F \right \rangle$ to $\mathcal{M}_v=\left \langle U,V,F_v \right \rangle$, where $F_v=\left \{ F \backslash f_i  \right \}\cup \left \{ V_i=v \right \}$. This is equivalent to removing $V_i$ from its original functional mechanism $v_i=f_i\left ( pa_i,u_i \right )$ and modifying this function to a constant function $V_i=v$. Formally, we denote the \textit{intervention} as $do(V_i=v)$, called the \textit{$do$-$calculus$}. It explores how causal mechanisms will change when external interventions, or experiments, are introduced. We denote the distribution after the intervention as $\mathbb{P}_{\mathcal{M}_{v}}(V)=\mathbb{P}(v_1,...,v_n\mid do(V_i=v))$, where
\begin{equation}
    \mathbb{P}(v_1,...,v_n\mid do(V_i=v))=
    \begin{cases}
    \prod_{j\neq i}\mathbb{P}(v_j\mid pa_j)& V_i=v \\
    0& V_i\neq v
    \end{cases}.
    \label{distribution}
\end{equation}

Therefore, if we want to calculate the causality of $V_i$ on $V_j$, we can use $\mathbb{P}(V_j|do(V_i))$ to replace the correlation $\mathbb{P}(V_j|V_i)$.

\textbf{Path and $d$-separation.} We recap two classic definitions \cite{pearl2016causal} to help us determine the independence between variables in the SCM graph. They enable us to avoid cumbersome probability calculations and instead obtain independence between variables directly from the graph.

\begin{definition}
    \label{path} \textbf{(Path)}
    In the SCM graph, the paths from variable $X$ to $Y$ include three types of structures: 1) chain structure: $A \rightarrow B\rightarrow C$ or $A \leftarrow B\leftarrow C$; 2) fork structure: $A \leftarrow B\rightarrow C$; 3) collider structure: $A \rightarrow B\leftarrow C$.
\end{definition}

\begin{definition}
    \label{separation} \textbf{($d$-separation)}
A path $p$ is blocked by a set of nodes $Z$ if and only if:
\begin{enumerate}
    \item $p$ contains a chain of nodes $A \rightarrow B\rightarrow C$ or a fork  $A \leftarrow B\rightarrow C$ such that the middle node $B$ is in $Z$, i.e., $A$ and $C$ are independent conditional on $B$, or
    \item $p$ contains a collider $A \rightarrow B\leftarrow C$ such that the collider node $B$ is not in $Z$, and no descendant of $B$ is in $Z$, i.e., $A$ and $C$ are independent without conditions.
\end{enumerate}
\end{definition}
If $Z$ blocks every path between two nodes $X$ and $Y$, then $X$ and $Y$ are \textit{$d$-separated}, conditional on $Z$, i.e., $X$ and $Y$ are independent conditional on $Z$, denoted as $X \upmodels Y \mid Z$.

\textbf{The Backdoor Criterion.} There is an \textit{empirical formula} for calculating the causation $\mathbb{P}(Y|do(X))$ of $X$ on $Y$, first proposed by \cite{pearl2011graphical}, known as the \textit{backdoor criterion}. Using the backdoor criterion, it is possible to determine which variables $Z$ should be conditioned on to find the causal relationship between any two variables $X$ and $Y$ in a causal model represented by a directed acyclic graph.

\begin{definition}
    \label{backdoor} \textbf{(Backdoor Criterion)}
Given an ordered pair of variables $(X, Y)$ in a directed acyclic graph, a set of variables $Z$ is said to satisfy the backdoor criterion with respect to $(X, Y)$ if: (1) $Z$ contains no descendant of $X$, and (2) $Z$ blocks every path between $X$ and $Y$ that contains an arrow into $X$.
\end{definition}

If a set of variables $Z$ satisfies the backdoor criterion with respect to $(X, Y)$, then the causation of $X$ on $Y$ can be computed using the following formula:
\begin{equation}
   \label{eq_backdoor}
    \mathbb{P}(Y=y|do(X=x))=\sum_z \mathbb{P}(Y=y|X=x,Z=z)\mathbb{P}(Z=z).
\end{equation}

\subsection{An Intuitive Example of Backdoor Adjustment}

To facilitate understanding, we provide an intuitive example of backdoor adjustment using a medical study scenario. Suppose we aim to determine whether a new drug positively affects recovery rates. We collect data on drug intake ($X$) and recovery outcomes ($Y$) for two subpopulations: men and women. As shown in Table \ref{tab:simpson-paradox}, within each subgroup, the drug appears beneficial (93\% vs. 87\% for men, 73\% vs. 69\% for women). However, when considering the entire population, non-drug users exhibit a higher recovery rate than drug users, which is the well-known Simpson's Paradox \cite{pearl2016causal}.

The cause of this paradox lies in the confounder (i.e., gender $Z$), which influences both drug intake and recovery rates. In the causal graph (Fig. \ref{fig:scm_s1}), this is represented as $ X \leftarrow Z \rightarrow Y$. Such a path is called the backdoor path, and directly estimating $P(Y\mid X)$ leads to misleading results due to spurious correlations.

Specifically, $P(Y=1\mid X=1)-P(Y=1\mid X=0)=0.78-0.83<0$, suggesting the drug has a negative effect.  However, when applying the Eq. \eqref{eq_backdoor} to obtain true causal effect, we have $P(Y=1|do(X=1)) - P(Y=1|do(X=0))=0.0502>0$, highlighting the importance of the backdoor adjustment to ensure accurate causal inference.

\begin{figure}
    \vspace{-0.3cm}
    \centering
    \includegraphics[width=0.6\linewidth,scale=1.00]{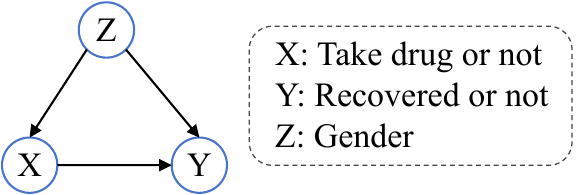}
    \vspace{-0.3cm}
    \caption{SCM graph of the drug study.}
    \vspace{-0.2cm}
    \label{fig:scm_s1}
\end{figure}

\begin{table}
    \centering
    \caption{Results of the drug study with gender being considered.}
    \vspace{-0.2cm}
    
    \adjustbox{max width=\linewidth}{
    \begin{tabular}{c|c|c}
    \toprule[1pt]
    ~ & Drug & No drug \\
    \midrule
    Men & 81 out of 87 recovered (93\%) & 234 out of 270 recovered (87\%) \\
    Women & 192 out of 263 recovered (73\%) & 55 out of 80 recovered (69\%) \\ 
    Combined data & 273 out of 350 recovered (78\%) & 289 out of 350 recovered (83\%) \\
    \bottomrule[1pt]
    \end{tabular}
    }
    \vspace{-0.3cm}
    \label{tab:simpson-paradox}
\end{table}

\subsection{Proof for Causal Effect Indentifiability}
\label{sec_identi}
With observable prompted wavelet subbands $P$, the SCM in Fig. \ref{fig_scm}(c) holds, and we employ the inference rules of do-calculus axioms to prove the identifiability of $P(Y|do(X))$ in Fig. \ref{fig_scm}(c) as follows.

\begin{theorem}[Rules of do Calculus]\label{docalculas}
    Let \(G\) be the directed acyclic graph associated with a causal model, and let \( P(\cdot) \) stand for the probability distribution induced by that model. For any disjoint subsets of variables \( X, Y, Z, \) and \( W \), we have the following rules.  
\begin{itemize}
    \item Rule 1 (Insertion/deletion of observations): 
    \begin{align*}
    &\text{if } (Y \perp\!\!\!\perp Z \mid X, W)_{G_{\overline{X}}},\\
    &P(y \mid do(x), z, w) = P(y \mid do(x), w). 
    \end{align*}
    \item Rule 2 (Action/observation exchange): 
    \begin{align*}
    &\text{if } (Y \perp\!\!\!\perp Z \mid X, W)_{G_{\overline{X}\underline{Z}}}, \\
    &P(y \mid do(x), do(z), w) = P(y \mid do(x), z, w). 
    \end{align*}
    \item Rule 3 (Insertion/deletion of actions): 
    \begin{align*}
    &\text{ if } (Y \perp\!\!\!\perp Z \mid X, W)_{G_{\overline{X},\overline{Z(W)}}},\\
    &P(y \mid do(x), do(z), w) = P(y \mid do(x), w), 
    \end{align*}
    where \( Z(W) \) is the set of \( Z \)-nodes that are not ancestors of any \( W \)-node in \(G_{\overline{X}} \).
\end{itemize}
\end{theorem}
\textbf{Proof:} We derive the closed-form solution of $P(Y|do(X))$ to prove its identifiability by using Theorem \ref{docalculas}. The figure \ref{fig:scm_q2} shows the subgraphs of Fig. \ref{fig_scm}(c) required for the do-calculus rules.

\begin{figure*} 
\vspace{-0.3cm}
    \centering
    \includegraphics[width=0.8\linewidth]{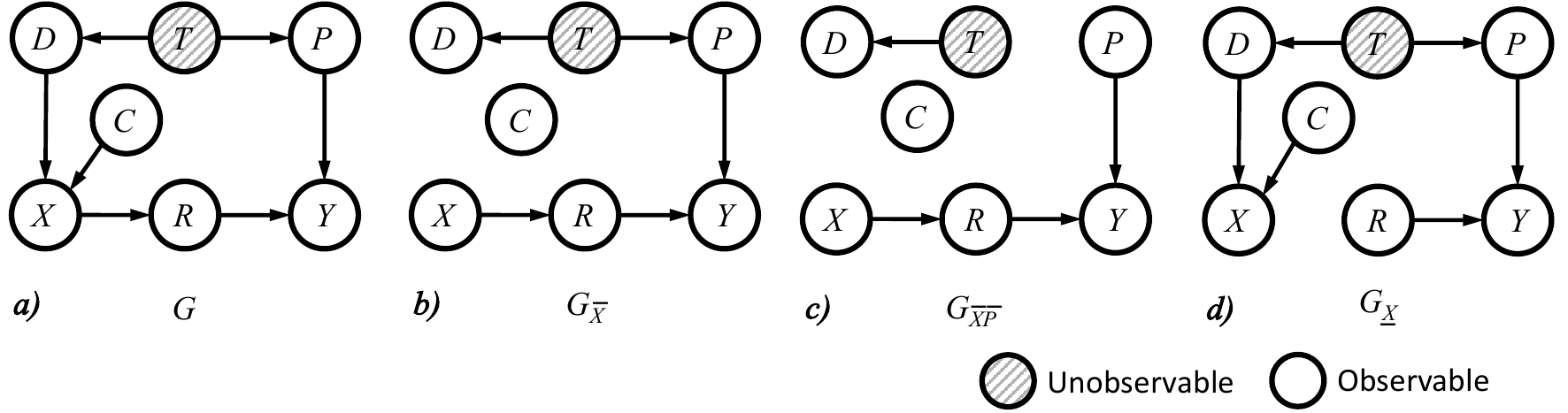}
    \vspace{-0.3cm}
    \caption{Subgraphs of G used in the derivation of causal effects.}
    \vspace{-0.2cm}
\label{fig:scm_q2}
\end{figure*}

Since $G_{\underline{X}}$ (see Fig. \ref{fig:scm_q2}(d)) contains a backdoor path from $X $ to $Y $: $X \leftarrow D \leftarrow T \rightarrow P \to Y $, we cannot apply Rule 2 to replace $do(x) $ with $x $. Naturally, we seek to block this path through variables on the path (such as $P$). This involves conditioning on and summing over all values of $P$:
\begin{equation}
    P(y | do(x)) = \sum_{p} P(y|p, do(x)) P(p|do(x)) \label{do}
\end{equation}
Now we need to handle two terms containing $do(x) $: $P(y|p, do(x)) $ and $P(p|do(x)) $. Since $P$ and $X $ are d-separated in $G_{\overline{X},\overline{P}} $ (see Fig. \ref{fig:scm_q2}(c)), the latter can be easily computed through the deletion action of Rule 3. Here we take $Z=P $, $X=X $, $W=\varnothing $:
\begin{equation}
    P(p|do(x)) = P(p) \quad \text{when}\  (P \perp\!\!\!\perp X)_{G_{\overline{X},\overline{P}}},
\end{equation}
To reduce the former term $P(y|p, do(x)) $, we apply Rule 2. Here we set $Z=X $, $W=P $:
\begin{equation}
    P(y|p, do(x)) = P(y|p, x)\quad\text{when} \ (X \perp\!\!\!\perp Y|P)_{G_{\underline{X},}}
\end{equation}
Note that in $G_{\underline{X}} $, $P$ d-separates $X $ from $Y $. This allows us to rewrite Equation \eqref{do} as follows:
\begin{equation}
    P(y|do(x)) = \sum_{p} P(y|p, x) P(p)  
\end{equation}
The right-hand side of the equation contains only functions of the observable distribution $P(y,x,p) $; therefore, $P(Y|do(X))$ is identifiable.

\subsection{Experimental Details for Causal Discovery}
\label{sec:expri_cd}
To examine whether the causal relationship between the degradation pattern $T$ and the semantic feature $C$ is consistent with the constructed SCM, we perform causal discovery on the real-world dataset corresponding to Fig.\ref{fig_sc}. The overall validation algorithm consists of two procedures.

\begin{figure} 
    \vspace{-0.3cm}
    \centering
    \includegraphics[width=0.45\linewidth]{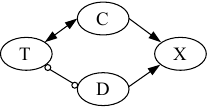}
    \vspace{-0.3cm}
    \caption{The output results of the FCI algorithm. Under the semantics of FCI output, a directed arrow $A\rightarrow B$ indicates that $A$ is a cause of $B$, a bidirected arrow $A \leftrightarrow B$ indicates that there is a latent confounder of $A$ and $B$, and a circle endpoint $A\ \text{o–o}\ B$ indicates that no set separates $A$ and $B$.} 
    \vspace{-0.3cm}
\label{FCI}
\end{figure}

\begin{table}
\centering
\caption{The accuracy of classifiers and the Fréchet-type distribution distance of generative models under different environments $e$.}
\label{ICP}

\begin{tabular}{c|cc|cc}%
\toprule 
\multirow{2}{*}{$e$} & \multicolumn{2}{c|}{Deraining} &  \multicolumn{2}{c}{Dehazing} \\
 & Acc(\%, $\uparrow$)& FID($\downarrow$) & Acc(\%, $\uparrow$)& FID($\downarrow$) \\
\midrule
1 & 43.25& 147.62& 52.38& 133.54\\
2& 27.84& 172.57& 30.86& 154.28\\
3& 29.59& 156.78& 38.74& 148.97\\
\bottomrule
\end{tabular}
\vspace{-0.3cm}
\end{table}

First, we apply the FCI algorithm, which outputs a PAG representing the Markov equivalence class of causal graphs consistent with the observed conditional independencies. Concretely, we run FCI on the variables $(T, C, D, X)$, using the Kernel Conditional Independence (KCI) test \cite{zhang2012kernel} as the conditional independence test to accommodate high-dimensional and nonlinear dependencies. The result of the FCI run, shown in Fig.\ref{FCI}, reveals a bidirectional edge between the degradation pattern $T$ and the semantic feature $C$ ($T \leftrightarrow C$), providing preliminary evidence for the existence of a spurious correlation induced by a latent confounder. 

Second, to further demonstrate that the latent confounder behind $C \leftrightarrow T$ is indeed the environment $E$, and to exclude any direct causal relationship between $C$ and $T$, we utilize the ICP principle \cite{peters2016causal}. ICP leverages causal mechanisms that remain invariant across different environments to distinguish true causal relations from spurious correlations caused by confounder. Specifically, we slightly repartition the dataset to obtain different environments(e.g., environment $e_1$ contains more rainy samples), and within each environment we train a degradation classifier and a conditional generative model to model $\hat{p}_e(T\mid C)$ and $\hat{p}_e(C\mid T)$ separately.
The experimental results in Table \ref{ICP} show a significant difference in $\hat{p}_e(T\mid C)$ and $\hat{p}_e(C\mid T)$ across different environments (measured by accuracy and FID, respectively), thereby confirming the absence of a direct $C \to T$ or $T \to C$ causal relationship. Combined with the output of FCI, the most plausible explanation for this strong association between $T$ and $C$ is that they are both influenced by the same environmental factor, i.e., $T \leftarrow E \to C$.

\end{document}